\documentclass[section]{article}%
\pdfoutput=1

\newif\ifarxiv
\arxivtrue%

\newif\ifshowack
\ifarxiv \showacktrue \else \showackfalse \fi

\PassOptionsToPackage{section}{placeins}

\ifarxiv
  \usepackage{labtemplate/aditi}
  \usepackage{labtemplate/utils}
  \usepackage{labtemplate/algorithm}
  \usepackage{labtemplate/algorithmic}
  \newif\ifcolmsubmission
  \colmsubmissionfalse
\else
  \usepackage[final]{colm2026_conference}
\fi

\usepackage{microtype}
\ifarxiv\else\usepackage{hyperref}\fi
\usepackage{url}
\usepackage{booktabs}
\ifarxiv\else
  \usepackage{amsmath}
  \usepackage{amssymb}
\fi
\usepackage{mathtools}
\usepackage{amsthm}
\usepackage{enumitem}
\usepackage{tcolorbox}
\usepackage{float}
\usepackage[section]{placeins}
\ifarxiv
\newtcolorbox{takeawaybox}{
  colback=gray!5,
  colframe=black,
  boxrule=0.8pt,
  arc=2pt,
  left=6pt,
  right=6pt,
  top=6pt,
  bottom=6pt
}
\else
\newtcolorbox{takeawaybox}{
  colback=gray!5,
  colframe=black,
  boxrule=0.8pt,
  arc=2pt,
  left=6pt,
  right=6pt,
  top=3pt,
  bottom=3pt,
  beforeafter skip balanced=4pt,
}
\fi

\usepackage{subcaption}

\usepackage{annotate-equations}
\usetikzlibrary{fit, positioning}
\definecolor{appleblue}{RGB}{95, 164, 211}
\definecolor{applegreen}{RGB}{81, 180, 102}
\definecolor{applered}{RGB}{255, 125, 136}
\definecolor{softblue}{RGB}{78,122,151}
\definecolor{softorange}{RGB}{160,106,70}
\definecolor{softred}{RGB}{185,95,104}
\definecolor{softgreen}{RGB}{45,130,64}
\theoremstyle{plain}
\theoremstyle{definition}
\theoremstyle{remark}

\ifarxiv
  \newcommand{\mainfigwidth}{\textwidth}
  \newcommand{\colmtight}{}
  \newcommand{\pipelinesize}{\large}
  \newcommand{\pipelinetopskip}{3.5em}
  \newcommand{\pipelinecapskip}{2.75em}
\else
  \newcommand{\mainfigwidth}{\textwidth}
  \newcommand{\colmtight}{\looseness=-1}
  \newcommand{\pipelinesize}{}
  \newcommand{\pipelinetopskip}{2.4em}
  \newcommand{\pipelinecapskip}{2.4em}
\fi

\ifarxiv\else
  \newtheorem{theorem}{Theorem}[section]
  \newtheorem{lemma}[theorem]{Lemma}

  \newtheorem{definition}[theorem]{Definition}
  \newtheorem{assumption}[theorem]{Assumption}

\fi

\usepackage{lineno}

\definecolor{darkblue}{rgb}{0, 0, 0.5}
\ifarxiv\else
  \hypersetup{colorlinks=true, citecolor=darkblue, linkcolor=darkblue, urlcolor=darkblue}
\fi

\ifarxiv
  \title{Early Data Exposure Improves Robustness to Subsequent Fine-Tuning}
\else
  \title{Early Data Exposure Improves \\ Robustness to Subsequent Fine-Tuning}
\fi

\author{Lawrence Feng, Gaurav R. Ghosal, Jacob Mitchell Springer, Ziqian Zhong \& Aditi Raghunathan \\
Carnegie Mellon University \\
\texttt{\{lawrencefeng,raditi\}@cmu.edu}
}

\ifarxiv
  \setauthors{Lawrence Feng \authorsep Gaurav R. Ghosal \authorsep Jacob Mitchell Springer \authorsep Ziqian Zhong \authorsep Aditi Raghunathan}
  \setaffils{Carnegie Mellon University}
  \setemail{\{lawrencefeng, raditi\}@cmu.edu}
  \setwebsite{https://ar-forum.github.io/earlyexposure-website/}
  \setwebsitedisplay{ar-forum.github.io/earlyexposure-website}
\fi

\begin{document}

\newcommand{\dweb}{\textcolor{softblue}{\mathcal{D}_{\textrm{pre}}}}
\newcommand{\dgen}{\textcolor{softblue}{\mathcal{D}_{\textrm{gen}}}}
\newcommand{\dspec}{\textcolor{softorange}{\mathcal{D}_{\textrm{spec}}}}
\newcommand{\dpost}{\textcolor{softgreen}{\mathcal{D}_{\textrm{post}}}}
\newcommand{\dft}{\textcolor{softred}{\mathcal{D}_{\textrm{ft}}}}

\newcommand{\pre}{\textcolor{softblue}{\theta_{\text{pre}}}}
\newcommand{\post}{\textcolor{softgreen}{\theta_{\text{post}}}}
\newcommand{\ft}{\textcolor{softred}{\theta_{\text{ft}}}}

\newcommand{\lft}{\mathcal{L}_{\mathrm{ft}}}
\newcommand{\lret}{\mathcal{L}_{\mathrm{ret}}}
\newcommand{\limm}{\mathcal{L}_{\mathrm{im}}}
\newcommand{\lpre}{\mathcal{L}_{\mathrm{pre}}}

\ifcolmsubmission
\linenumbers
\fi

\maketitle

\begin{abstract}
How can we train models whose post-trained capabilities survive subsequent fine-tuning? Rather than focusing on downstream interventions to mitigate forgetting of upstream capabilities, we study how upstream training choices --- that is, the manner in which a capability is acquired --- shape how robustly that capability is retained. We investigate this question in a controlled three-stage language-model pipeline: pretraining, post-training to acquire a target capability, and downstream fine-tuning on a new objective. Across 135M and 1B models, two post-training domains, and two downstream fine-tuning tasks, we find that immediate post-training performance does not reliably predict retention after subsequent fine-tuning: training recipes that look equivalent immediately after post-training can retain the target capability very differently after subsequent fine-tuning. In particular, \textit{early exposure} --- mixing post-training data into pretraining --- consistently improves the frontier between retained upstream performance and downstream performance. In compute-matched experiments, where the target data must be allocated between pretraining and post-training, we find that the optimum lies at neither extreme. Together with our other empirical and theoretical findings, this supports the view that post-training drives immediate specialization while early exposure improves robustness to later forgetting. Replay and dropout, typically used to mitigate forgetting as it occurs during fine-tuning, provide complementary gains to early exposure when applied during post-training. Our findings suggest that robustness to subsequent fine-tuning should be treated as a first-class objective of upstream training, addressed preventatively through choices like early exposure rather than reactively during fine-tuning itself.

\end{abstract}

\section{Introduction}

When a post-trained language model is released for downstream fine-tuning, its carefully acquired capabilities are at risk. Fine-tuning on a new objective routinely causes catastrophic forgetting of behaviors introduced during post-training — whether instruction following, domain knowledge, coding ability, or safety-related behavior \citep{yang2025qwen3technicalreport, olmo2025olmo3}.

Most prior work treats this as a problem for the downstream fine-tuner to solve. If fine-tuning degrades prior capabilities, the natural response is to modify that fine-tuning stage: replay earlier data~\citep{bethune2025scalinglawsforgettingfinetuning, kotha2026replayingpretrainingdataimproves}, regularize the update~\citep{Kirkpatrick_2017}, restrict the trainable parameters~\citep{hu2021loralowrankadaptationlarge, biderman2024loralearnsforgets}, or jointly optimize competing objectives \citep{wortsman2022modelsoupsaveragingweights, wortsman2022robustfinetuningzeroshotmodels}.

We take a complementary view: robustness to subsequent fine-tuning should be treated as an objective of upstream model development. Upstream developers typically train in two stages — first on a large general corpus to build broad language understanding, then on a smaller, often scarce, targeted dataset to instill specific capabilities. Because this second stage uses limited data, how and when it is used matters. Our central intuition is that how a model learns a capability shapes how robustly it is retained: two models that reach identical post-training performance can differ substantially in how well those capabilities survive later adaptation.

To study this question, we use a controlled three-stage pipeline reflecting this standard practice: an upstream developer first {\color{softblue}pretrains} on a broad corpus, then {\color{softgreen}post-trains} on a smaller targeted dataset to acquire specific capabilities, and finally hands the resulting model to a downstream user who {\color{softred}fine-tunes} it on a new objective (Figure \ref{fig:pipeline}). We study this framework across multiple controlled settings spanning different post-training and downstream fine-tuning regimes (Table~\ref{tab:instantiations}), including both domain and behavioral adaptation, and evaluate these settings for 135M parameter models, extending our findings to 1B parameter models. We hold the downstream fine-tuning method fixed, applying standard supervised fine-tuning, and sweep its learning rate to characterize how upstream choices shape the tradeoff between downstream performance, retention of the post-trained capability, and performance on the broader pretraining distribution. Accordingly, our evaluation centers on the tradeoff frontier induced by different methods, rather than immediate post-training performance alone.

\begin{figure}[t]
\vspace{\pipelinetopskip}%
{\pipelinesize
\begin{equation*}
  \eqnmarkbox[appleblue]{pre}{\theta_{\mathrm{pre}}}
  \;\longrightarrow\;
  \eqnmarkbox[applegreen]{post}{\theta_{\mathrm{post}}}
  \hspace{11em}%
  \eqnmarkbox[applered]{ft}{\theta_{\mathrm{ft}}}
  \label{eq:pipeline}
\end{equation*}
}
\annotate[yshift=-1.8em]{below, left}{pre}{\normalsize Pretraining}
\annotate[yshift=-1.8em]{below, right}{post}{\normalsize Post-training for $X$}
\annotate[yshift=-1.8em]{below, right}{ft}{\normalsize Fine-tuning for task $Y$}
\begin{tikzpicture}[overlay, remember picture]
  \begin{scope}[on background layer]
    \node[draw=gray!60, rounded corners=5pt,
          fit=(pre)(post),
          inner xsep=14pt, inner ysep=14pt] (leftbox) {};
    \node[draw=gray!60, rounded corners=5pt,
          fit=(ft),
          inner xsep=14pt, inner ysep=14pt] (rightbox) {};
  \end{scope}
  \draw[->, dashed, thick, gray!60]
    (leftbox.east) --
    node[midway, below, font=\small\sffamily, text=gray!80]
      {downstream user}
    (rightbox.west);
  \node[above=0.4em of leftbox, font=\small\sffamily]  (leftlabel)
    {Acquiring Capability $X$};
  \node[anchor=base, font=\small\sffamily] at (rightbox.north |- leftlabel.base)
    {\vphantom{g} Is $X$ Retained?};
\end{tikzpicture}
\vspace{\pipelinecapskip}%
\caption{Overview of our three-stage experimental setup, in contrast to a typical two-stage setup. A first party pretrains then post-trains a model with the goal of achieving high performance on domain $X$. Subsequently, downstream users fine-tune $\post$ for a task $Y$, causing catastrophic forgetting of domain $X$. Previous work investigates interventions in the third stage: how can we fine-tune for $Y$ while mitigating forgetting on $X$? In this work, we investigate how \textit{the way in which} $X$ is learned affects how it is (or not) forgotten.} 
\label{fig:pipeline}
\vspace{-1em}
\end{figure} 

Our main intervention is simple: we expose the model to some of the eventual post-training data earlier by mixing it in during pretraining. Across datasets and model sizes, we find that this \textit{early exposure} improves the tradeoff between retained upstream capability and downstream fine-tuning loss (Figure \ref{fig:four_pipeline_frontiers}), even when it has little or no visible effect on immediate post-training performance (Section \ref{subsec:latent-mixing}).

Why does early exposure help? Our theoretical account (Section \ref{sec:theory}) suggests that mixing during pretraining allows the post-training capability to be represented in specialized features that are less vulnerable to subsequent interference. Our compute-matched experiments reinforce this view: even under a fixed budget of post-training data, the optimum does not lie at either extreme of allocating all exposure to pretraining or all exposure to post-training, but between the two (Section \ref{subsec:compute-matched}).

If the manner of learning matters, a natural follow-up question arises: what other interventions can shape how a capability is acquired? We study replay and dropout from this perspective (Section \ref{subsec:dropout-replay}). Intuitively, replay mitigates forgetting by interleaving earlier data with new-domain updates, while dropout discourages co-adaptation and promotes more robust representations \citep{rolnick2019experiencereplaycontinuallearning, 10.5555/2627435.2670313}. Importantly, we evaluate whether the benefits of these interventions persist after a later downstream fine-tuning stage. We find that both improve the tradeoff frontier while remaining complementary to pretraining-time mixing.

Together, our findings show that the upstream training process is a meaningful lever for shaping robustness to subsequent fine-tuning. In particular, a remarkably simple intervention---mixing a small amount of post-training data into pretraining---can materially improve how well capabilities survive subsequent fine-tuning. Replay and dropout provide additional complementary gains, further suggesting that robustness can be influenced well before forgetting is observed downstream. More broadly, our results point to a promising direction for future work: building models that are inherently easier to adapt by designing the upstream training pipeline to make valuable capabilities more durable from the start.\ifarxiv\else{} We defer an extended discussion of related work to Appendix~\ref{sec:related_works}.\fi

\begin{figure*}[t]                                
      \centering                                  
      \includegraphics[width=\mainfigwidth]{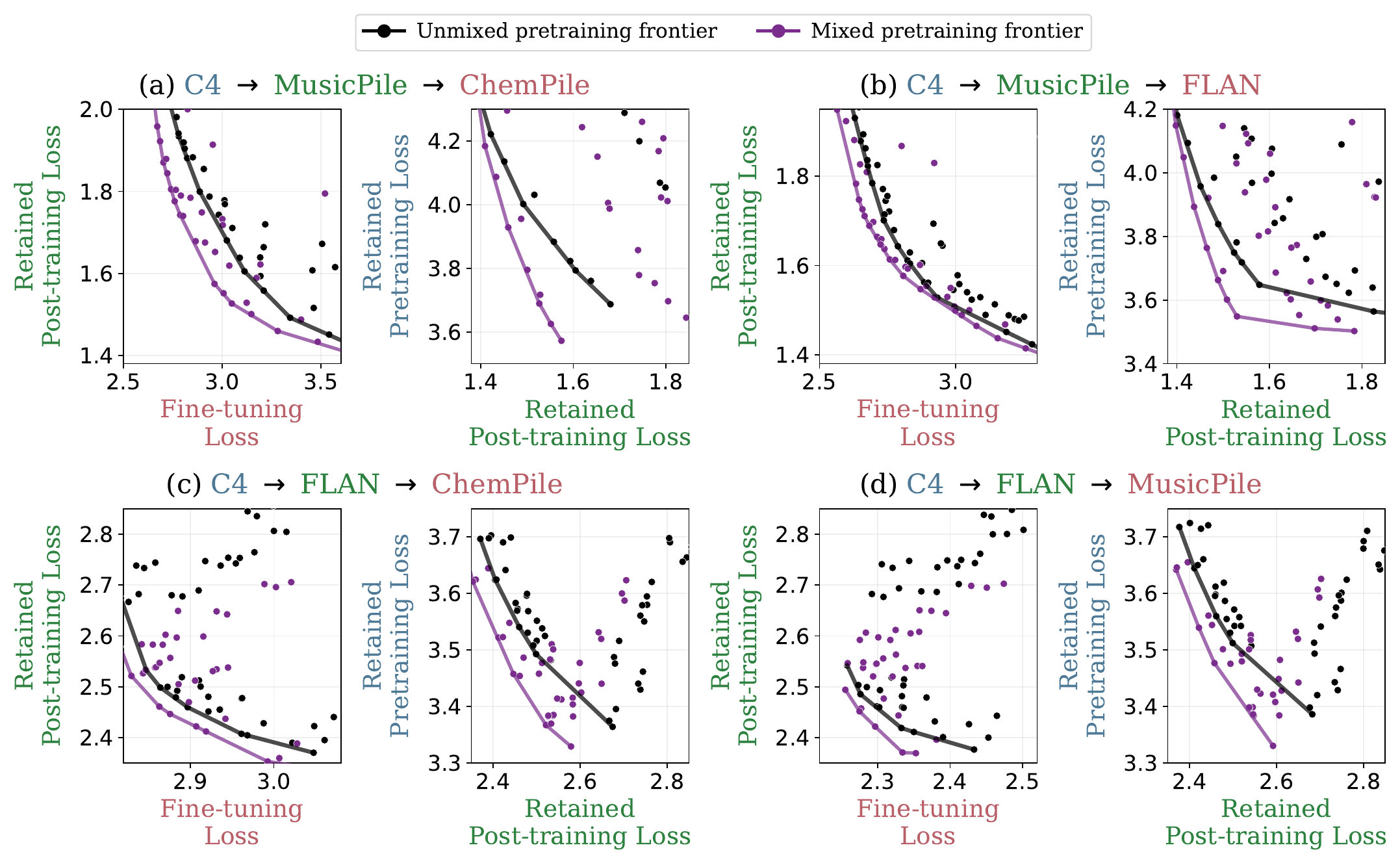}
      \caption{\textbf{Mixing during pretraining improves the frontier across four training pipelines (135M).} Each panel corresponds to one 
  3-stage pipeline. Within each panel, the left plot shows retained post-training loss versus downstream fine-tuning loss, and the right plot     
  shows retained pretraining loss versus retained post-training loss. \textbf{Black} denotes the frontier obtained from unmixed pretraining, and  
  \textbf{\color{violet}purple} denotes the frontier obtained from mixed pretraining. Across all four pipelines, mixing shifts the      
  frontier toward lower retained post-training loss, lower retained pretraining loss, and lower downstream fine-tuning loss,
  indicating that early exposure to post-training data can improve its downstream retention after subsequent fine-tuning.}
      \label{fig:four_pipeline_frontiers}
  \end{figure*}

\ifarxiv
\section{Related Works}
\label{sec:related_works}

\textbf{Catastrophic Forgetting.} A recurring challenge in sequential training is \textit{catastrophic forgetting}: when a model is optimized on new data, its performance can deteriorate on behaviors it previously exhibited \citep{mccloskey1989catastrophic}. For language models, this phenomenon shows up in modern training pipelines. For example, instruction tuning and RLHF can trade off against preexisting capabilities, an effect often discussed as an ``alignment tax'' \citep{ouyang2022traininglanguagemodelsfollow}. Relatedly, several works show that behaviors introduced during safety fine-tuning can be quickly weakened or reversed by subsequent training on different objectives or data \citep{yang2023shadowalignmenteasesubverting, qi2023finetuningalignedlanguagemodels}. These tradeoffs also appear in adjacent settings such as knowledge editing \citep{nishi2025representationshatteringtransformerssynthetic} and unlearning \citep{maini2024tofutaskfictitiousunlearning}. Beyond documenting the effect, recent work has started to map how training choices shape its severity: for instance, LoRA-style adaptation can alter forgetting dynamics \citep{biderman2024loralearnsforgets}, and longer pretraining can change how brittle or persistent acquired capabilities are \citep{springer2025overtrainedlanguagemodelsharder}. In this paper, we focus on catastrophic forgetting of post-trained capabilities, and study what properties of an intermediate checkpoint determine whether capabilities persist under subsequent training.

\textbf{Data Placement in Pretraining.} A line of recent works examine pretraining interventions for enforcing desired downstream capabilities and properties.  \citet{maini2025safetypretraininggenerationsafe,obrien2025deepignorancefilteringpretraining} propose filtering and augmenting data during pretraining to improve safety. Similarly, \citet{sam2026introducesafetyinterventionspretraining} demonstrate that the impact of such interventions improves as they are introduced earlier in pretraining. While these works incorporate downstream tasks during pretraining, they extensively modify the pretraining corpus by incorporating data-augmentations and filtering of the dataset. \citet{baek2026finetunersfallacypretrainfinetuning} demonstrate that mixing post-training data during pretraining can immediately improve in-domain performance relative to simply fine-tuning. In our work, we introduce an additional benefit of early exposure to post-training data: robustness to catastrophic forgetting during future training.

\fi

\section{Preliminaries and Setting}

Typically, a model developer {\color{softblue}(1) pretrains} a model on a general web corpus and {\color{softgreen}(2) post-trains} the model on a target domain; downstream users then {\color{softred}(3) fine-tune} the model for their own purposes. Throughout this work, we'll describe stages {\color{softblue}(1)} and {\color{softgreen}(2)} as upstream relative to stage {\color{softred}(3)} controlled by the downstream end user. Each stage is associated with a dataset: $\dweb$ (general pretraining), $\dpost$ (post-training), and $\dft$ (fine-tuning). The post-training corpus is assumed to be much smaller than the pretraining web corpus, reflecting the practical regime where post-training data is relatively scarce.

We write $\mathcal{L}(\theta;\mathcal{D})$ for the loss of parameters $\theta$ evaluated on dataset $\mathcal{D}$. All losses are computed on held-out splits of the corresponding datasets.

\textbf{Stage 1: Upstream Pretraining.}
The model is first pretrained on a general web corpus
$\dweb$, producing a pretrained checkpoint
$\pre$. In some experiments, the upstream developer additionally mixes a fraction $\lambda \in [0,1]$ of the post-training corpus
$\dpost$ into this stage. Here, $\lambda=0$ denotes no exposure to $\dpost$ during pretraining, while $\lambda>0$ denotes \textit{early exposure} to the post-training dataset. In this work, $\lambda$ is at most one, denoting at most one pass over $\dpost$ during pretraining.

\textbf{Stage 2: Upstream Post-training.}
Starting from $\pre$, the upstream developer post-trains the model on a relatively smaller corpus $\dpost$ to acquire a target capability or domain adaptation, yielding the post-trained checkpoint $\post$. We measure performance on $\dpost$ immediately after this stage using the \emph{immediate post-training loss}
\[
\mathcal{L}_{\mathrm{im}} := \mathcal{L}(\post;\dpost).
\]

\textbf{Stage 3: Subsequent Fine-Tuning.}
A downstream user then fine-tunes the post-trained model on a new objective $\dft$, producing the checkpoint $\ft$. We measure performance on this new objective using the \emph{downstream fine-tuning loss}
\[
\mathcal{L}_{\mathrm{ft}} := \mathcal{L}(\ft;\dft).
\]

Subsequent fine-tuning can degrade capabilities acquired during upstream post-training. To measure how much of the post-trained capability survives, we also evaluate $\ft$ on $\dpost$, defining the \emph{retained post-training loss}
\[
\lret := \mathcal{L}(\ft;\dpost).
\]

Our central question is how upstream training choices affect downstream adaptation, retention of post-trained capabilities, and preservation of general-domain performance under subsequent fine-tuning.

\begin{table}[t]
\begin{center}
\small
\begin{tabular}{llll}
\toprule
\textbf{Pipeline} & $\dweb$ & $\dpost$ & $\dft$ \\
\midrule
Music $\rightarrow$ Chemistry & C4 & MusicPile & ChemPile \\
Music $\rightarrow$ Instruction & C4 & MusicPile & FLAN \\
Instruction $\rightarrow$ Chemistry & C4 & FLAN & ChemPile \\
Instruction $\rightarrow$ Music & C4 & FLAN & MusicPile \\
\bottomrule
\end{tabular}
\end{center}
\caption{Experimental instantiations of the three-stage pipeline. We vary the upstream post-training corpus $\dpost$ and downstream fine-tuning corpus $\dft$ while keeping the general pretraining corpus $\dweb$ fixed to C4.}
\label{tab:instantiations}
\end{table}

\subsection{Evaluation methodology}
\label{subsec:evaluation-methodology}

Subsequent fine-tuning is inherently multi-objective. A downstream user may care not only about performance on the new fine-tuning objective $\dft$, but also about retaining capabilities acquired during upstream post-training on $\dpost$ and preserving more general capabilities associated with the pretraining distribution $\dweb$. We therefore track three losses throughout this work: the downstream fine-tuning loss $\lft$, the retained post-training loss $\lret$, and the retained pretraining loss
\[
\lpre := \mathcal{L}(\ft;\dweb).
\]

We use validation loss as our evaluation metric. Prior work has established loss as a reliable, scale-invariant proxy for capability: models with matched pretraining loss exhibit equivalent downstream task performance \citep{duunderstanding, gadre2024languagemodelsscalereliably, chen2025scalinglawspredictingdownstream}. This is particularly important at our training scales, where task accuracy is noisy and near random chance; in contrast, loss provides a continuous and smooth signal that enables fine-grained comparisons between models that may appear similar under coarse or discrete metrics.

Sweeping upstream training choices and hyperparameters yields checkpoints with different tradeoffs among these objectives. We summarize the best attainable tradeoffs using 2D \emph{loss frontiers}: for each method, we plot the Pareto-optimal set of checkpoints in a given 2D projection, i.e., those for which no other checkpoint from the same method achieves lower loss on both axes simultaneously. Our main analysis uses two complementary views: $(\lret, \lft)$, which captures the tradeoff between retaining the post-trained capability and adapting to the downstream task, and $(\lpre, \lret)$, which captures the tradeoff between retaining broader pretraining capabilities and retaining the post-trained capability. Together, these views provide interpretable slices through the underlying three-objective tradeoff.

\subsection{Experimental instantiations}

Across experiments, we fix the general pretraining corpus to C4 and study the four three-stage instantiations shown in Table~\ref{tab:instantiations}. These settings cover two qualitatively different forms of upstream capability acquisition, domain adaptation and instruction tuning, and let us test whether the same robustness phenomena appear across different downstream fine-tuning objectives.

\subsection{Model scales and training budgets}
\label{subsec:training-details}

Unless otherwise stated, we use a SmolLM2-style architecture~\citep{allal2025smollm2smolgoesbig} at two scales: our primary experiments use a 135M-parameter model, and we additionally run a 1B-parameter variant to test whether the same qualitative patterns persist at larger scale. For the 135M experiments, we pretrain on approximately 10B tokens from $\dweb$, optionally with early exposure to $\dpost$ during Stage~1. Starting from the resulting checkpoint $\pre$, we perform Stage~2 post-training on $\dpost$ using AdamW with linear warmup and cosine decay. In all but the compute-matched experiments, training proceeds exclusively on $\dpost$, with no restriction on dataset repetitions: we apply early stopping and continue training as long as validation loss on $\dpost$ improves (up to a maximum budget of 2B tokens). This ensures that all models are trained to convergence on $\dpost$, and that differences in downstream retention are not attributable to unequal training duration. We then fine-tune each post-trained checkpoint $\post$ on $\dft$ for a fixed token budget of 200M tokens with various learning rates.

Full optimizer settings, sweep ranges, and per-intervention details are provided in Appendix~\ref{appendix:training}.

\section{Experiments and Results}
\label{sec:mainresults}

We begin by asking whether pretraining-time mixing has any effect once post-training is run to convergence (Section \ref{subsec:latent-mixing}). We then ask whether scarce post-training data should be mixed during pretraining or reserved for a dedicated post-training stage (Section \ref{subsec:compute-matched}). Finally, we ask whether the benefits of mixing persist across broad hyperparameter sweeps and multiple pipeline instantiations (Section \ref{subsec:frontier-main}), before turning to replay and dropout as complementary post-training interventions (Section \ref{subsec:dropout-replay}).

\begin{figure}[t]
    \centering
    \begin{subfigure}[t]{0.49\textwidth}
        \centering
        \includegraphics[width=\linewidth]{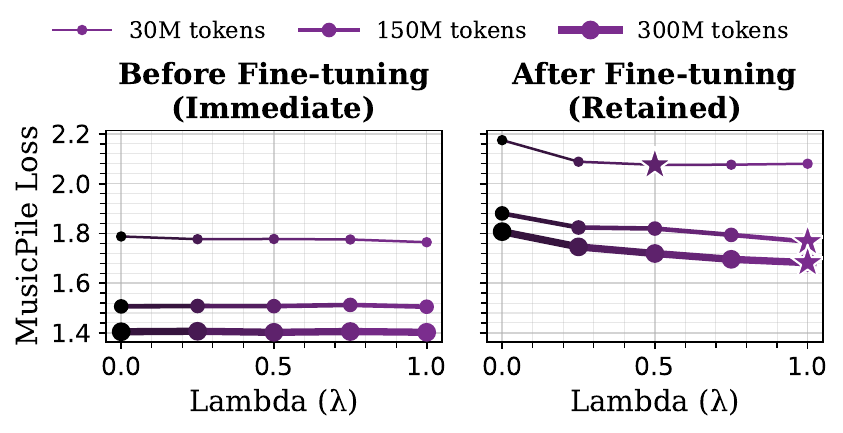}
        \caption{Varying mixture fraction $\lambda$.}
        \label{fig:impactmixingratio_saturation}
    \end{subfigure}
    \hfill
    \begin{subfigure}[t]{0.49\textwidth}
        \centering
        \includegraphics[width=\linewidth]{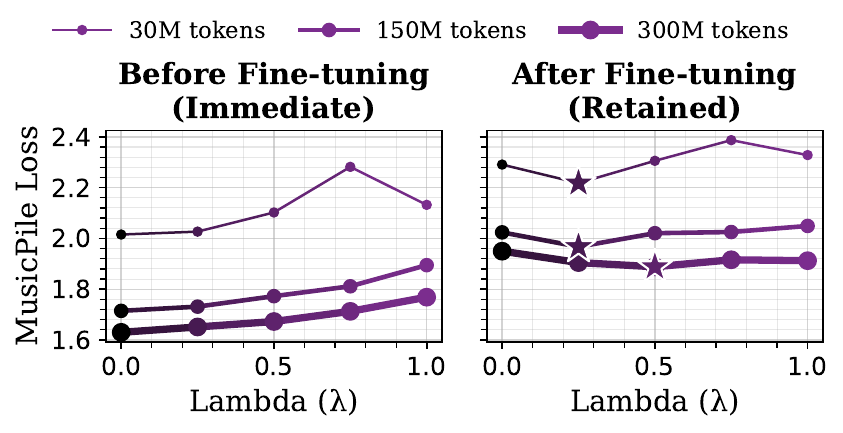}
        \caption{Compute-matched setting.}
        \label{fig:computematched}
    \end{subfigure}
    \caption{
    \textbf{Left:} As the mixture fraction $\lambda$ increases, immediate MusicPile loss after post-training remains nearly constant, while retained MusicPile loss after downstream fine-tuning on ChemPile improves. This shows that the benefits of mixing can be \textit{latent}: they may not be visible immediately after post-training, but emerge after subsequent fine-tuning.
    \textbf{Right:} In a compute-matched setting where total MusicPile exposure is held fixed across pretraining and post-training, increasing $\lambda$ worsens immediate MusicPile loss after post-training but improves retained MusicPile loss after downstream fine-tuning. Thus, even under a fixed MusicPile token budget, allocating some exposure earlier in training yields better retention.}
    \label{fig:latent-mixing-and-compute-matched}
\end{figure}

\subsection{Immediate post-training performance does not reflect downstream retention}
\label{subsec:latent-mixing}

We begin with a controlled study of the mixing ratio $\lambda$, designed to isolate whether early exposure to $\dpost$ has any effect once upstream post-training is run to convergence using early stopping. Unlike the broad hyperparameter sweeps used in our main frontier analysis, these experiments fix the Stage~2 post-training procedure and vary only how much of $\dpost$ is seen during Stage~1 pretraining.

\textbf{Setup.}
We fix the post-training configuration and vary only the pretraining mixing fraction $\lambda \in \{0, 0.25, 0.5, 0.75, 1.0\}$. We study three post-training dataset sizes $|\dpost| \in \{30\text{M},150\text{M},300\text{M}\}$ where $\dpost\subset\text{MusicPile}$. Starting from each pretrained checkpoint, we post-train on $\dpost$ until convergence using a fixed hyperparameter configuration. To induce forgetting, we then fine-tune on $\dft\subset\text{ChemPile}$, and report both the immediate post-training loss $\limm$ and the retained post-training loss $\lret$ at a fixed Stage~3 learning rate of $5\times 10^{-5}$ (Figure~\ref{fig:impactmixingratio_saturation}).

\textbf{Result.}
Varying $\lambda$ has little effect on immediate post-training performance: across dataset sizes, $\limm$ remains nearly flat as the mixing fraction increases. In other words, once post-training is allowed to run to convergence, mixed and unmixed models reach similar performance on $\dpost$. However, these checkpoints behave very differently under subsequent fine-tuning. As $\lambda$ increases, the retained post-training loss $\lret$ consistently decreases, indicating that models with more exposure to $\dpost$ during pretraining forget less after subsequent downstream adaptation. 

\begin{takeawaybox}
\textbf{Takeaway.} Early exposure can substantially improve retention under subsequent fine-tuning even when it provides little or no benefit to immediate post-training performance.
\end{takeawaybox}

\subsection{Early exposure and post-training play different roles under a fixed data budget}
\label{subsec:compute-matched}

The experiment above in Section \ref{subsec:latent-mixing} shows that early exposure to post-training data can improve its downstream retention even when it has little effect on immediate post-training performance. However, those experiments do not isolate whether the benefit comes from \emph{when} $\dpost$ is introduced or simply from the model seeing more total $\dpost$ tokens. We therefore ask a more controlled question: under a fixed $\dpost$ budget, should post-training data be mixed into pretraining at all, or is it better reserved for a dedicated post-training stage?

\textbf{Setup.}
We fix the total number of $\dpost$ tokens seen across Stage~1 pretraining and Stage~2 post-training and vary only how that budget is allocated. For each mixing fraction $\lambda \in \{0, 0.25, 0.5, 0.75, 1.0\}$, we expose the model to a $\lambda$-fraction of $\dpost$ during pretraining and reserve the remaining $(1-\lambda)$ fraction for post-training, so every model sees exactly one pass over $\dpost$ in total. Thus, $\lambda=0$ assigns the full budget to dedicated post-training, while $\lambda=1$ assigns it entirely to mixed pretraining. We evaluate this allocation study with $\dpost\subset\text{MusicPile}$ at three dataset sizes, $|\dpost| \in \{30\text{M},150\text{M},300\text{M}\}$, use $\dft\subset\text{ChemPile}$ for downstream fine-tuning, and report both $\limm$ and $\lret$ after Stage~3 fine-tuning at a fixed learning rate of $5 \times 10^{-5}$ (Figure~\ref{fig:computematched}).

\textbf{Result.}
Figure~\ref{fig:computematched} reveals a clear tradeoff. As $\lambda$ increases, the immediate post-training loss $\limm$ worsens: under a fixed $\dpost$ budget, allocating more post-training data to pretraining leaves fewer tokens for the post-training stage that exclusively optimizes for $\dpost$. The downstream picture is different. Increasing $\lambda$ consistently improves retained performance after Stage~3 fine-tuning, lowering $\lret$ across dataset sizes. As a result, the best immediate post-training performance is achieved by allocating all of $\dpost$ to Stage~2, while the best retained performance after downstream adaptation is achieved at a nonzero mixture fraction. Even under a fixed data budget, the optimum therefore lies between the two extremes of all-post-training and all-pretraining allocation. This suggests that dedicated post-training and early exposure are doing something meaningfully different: concentrating $\dpost$ in Stage~2 yields stronger immediate fitting to the post-training domain, while exposure during pretraining makes that capability less brittle under later training. We provide a theoretical understanding of this in Section \ref{sec:theory}.

\begin{takeawaybox}
\textbf{Takeaway:} Under a fixed $\dpost$ budget, the best immediate post-training performance occurs when all data is reserved for Stage~2, but the best retained performance after downstream fine-tuning occurs at a positive mixture fraction.
\end{takeawaybox}

\subsection{Early exposure improves the loss frontier across hyperparameter sweeps}
\label{subsec:frontier-main}

The controlled studies above isolate two key phenomena: the benefit of mixing can be latent, and under a fixed post-training-data budget the best retained performance is achieved at a nonzero mixture fraction. We now ask whether these conclusions persist once we move beyond controlled comparisons to the more realistic setting in which both post-training and downstream fine-tuning offer many tunable degrees of freedom. In practice, an upstream developer can vary post-training hyperparameters to reach different tradeoffs between adaptation and retention, while a downstream end user may likewise vary fine-tuning hyperparameters to target different operating points. We therefore evaluate each upstream strategy not by a single checkpoint, but by the frontier of attainable checkpoints it induces across broad Stage~2 post-training sweeps and Stage~3 fine-tuning sweeps.

\textbf{Setup.}
For each pipeline in Table~\ref{tab:instantiations}, we sweep Stage~2 post-training hyperparameters under both unmixed and mixed pretraining, then fine-tune every resulting checkpoint on the downstream objective using a range of Stage~3 learning rates, and evaluate the paired frontier views shown in Figure~\ref{fig:four_pipeline_frontiers}. Within each pipeline, the left panel plots retained post-training loss against downstream fine-tuning loss, while the right panel plots retained pretraining loss against retained post-training loss.

\textbf{Result.}
Across all four pipelines, early exposure consistently shifts the frontier relative to unmixed pretraining. In the $(\lret,\lft)$ view, mixing yields lower retained post-training loss at matched downstream fine-tuning loss, indicating greater robustness of the post-trained capability under subsequent adaptation. In the $(\lpre, \lret)$ view, mixing also improves the tradeoff between preserving broader pretraining capabilities and preserving the post-trained capability. These gains appear across both domain and behavioral post-training settings, suggesting that the benefit of mixing is not confined to a single dataset pair or narrow training regime. We further find that a similar qualitative frontier improvement appears in our 1B experiments, suggesting that the benefit of mixing persists beyond the small-model setting (Figures \ref{fig:1b-mixed-mp}
and \ref{fig:1b-mixed-flan-chempile}).

\begin{takeawaybox}
\textbf{Takeaway:} Across hyperparameter sweeps and training pipelines, early exposure consistently improves the attainable tradeoffs among downstream fine-tuning performance, retained post-training performance, and retained pretraining performance.
\end{takeawaybox}

\subsection{Replay and dropout provide complementary gains}
\label{subsec:dropout-replay}

\begin{figure}[t]
      \centering
      \includegraphics[width=\mainfigwidth]{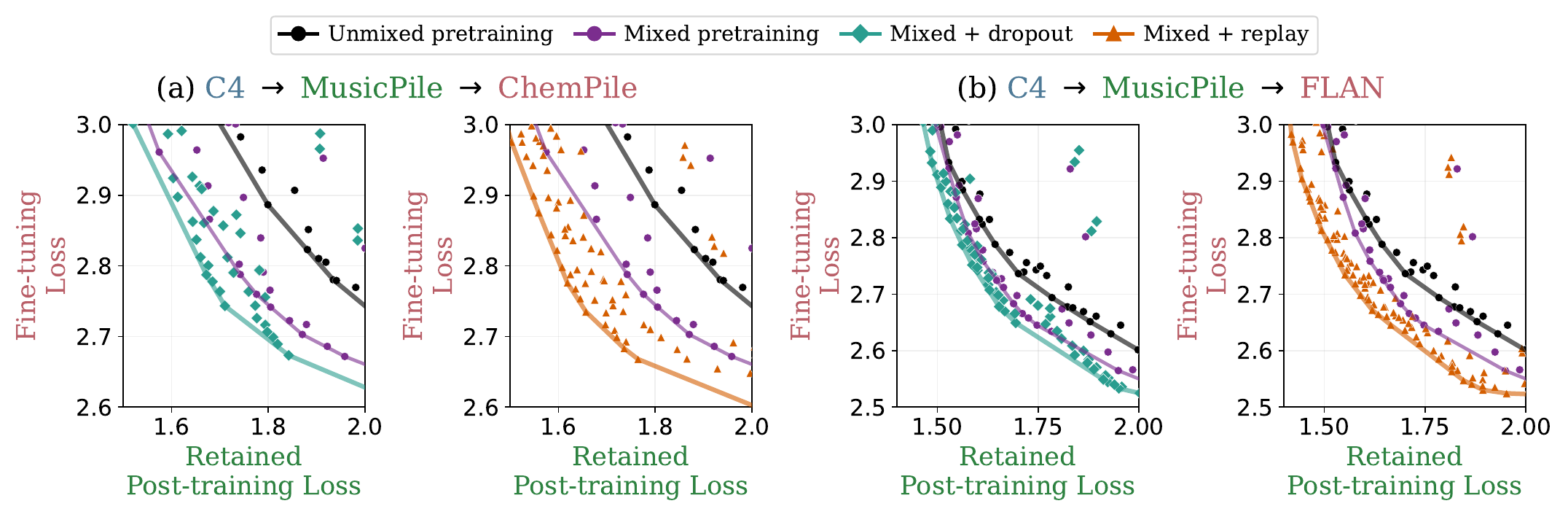}
    \caption{\textbf{Replay and dropout provide complementary gains on top of mixed pretraining.}
Each subfigure shows one 3-stage pipeline. Within each subfigure, the \textbf{left} panel compares unmixed pretraining, mixed pretraining, and mixed pretraining + dropout, while the \textbf{right} panel compares unmixed pretraining, mixed pretraining, and mixed pretraining + replay. Across both downstream settings, adding dropout or replay to mixed pretraining further shifts the loss frontier, indicating that these post-training interventions provide complementary gains rather than replacing the effect of pretraining-time mixing.}
    \label{fig:mixing_combined}
\end{figure}

Mixing $\dpost$ into pretraining is one way to shape how the model acquires the target post-training capability, but it is not the only one. We next study two alternative upstream interventions that act during post-training itself: \emph{replay}, which mixes a small amount of general-domain data into post-training, and \emph{dropout}, which regularizes the post-training update. We view both as interventions on \emph{how} the model learns $\dpost$, rather than simply how much performance it achieves immediately after post-training.

\textbf{Setup.}
We evaluate replay and dropout in the same three-stage framework as above, using broad Stage~2 hyperparameter sweeps and, for every resulting checkpoint, sweeping Stage~3 learning rates before measuring the resulting frontier between downstream fine-tuning loss $\lft$ and retained post-training loss $\lret$. Replay is implemented by mixing a small fraction (1\% following \citet{bethune2025scalinglawsforgettingfinetuning}) of general-domain data from $\dweb$ into post-training, while dropout is applied during post-training as a regularizer. At 1B we additionally test a larger (10\%) replay fraction (Appendix~\ref{app:1b-dropout-replay}).

\textbf{Result.}
Figure~\ref{fig:mixing_combined} shows that both replay and dropout further improve the loss frontier relative to early exposure alone. Replay encourages the model to acquire the post-training domain without simply overwriting broader pretraining features, while dropout may promote more distributed and robust representations during Stage~2 learning. These gains persist after downstream fine-tuning in both representative pipelines we study, indicating that post-training interventions can meaningfully improve robustness to later adaptation.

At 1B, the effect of replay on the $(\lft,\lret)$ frontier is weaker. However, replay remains useful for preserving broader general-domain performance, suggesting that its benefits may shift across scales and objectives (figures in Appendix \ref{app:1b-dropout-replay}). 

\textbf{Complementarity with mixing.}
Neither replay nor dropout eliminates the value of pretraining-time mixing. Instead, both remain complementary to mixing: the strongest frontiers are obtained by combining post-training interventions with mixed pretraining. This reinforces the broader picture developed above. Pretraining-time mixing plays a distinct role by changing when the model first encounters the post-training capability, while replay and dropout shape how that capability is learned during Stage~2.

\begin{takeawaybox}
\textbf{Takeaway:} Replay and dropout are upstream alternatives to mixing during pretraining that can also improve robustness to subsequent fine-tuning, and their gains are complementary with early exposure. 
\end{takeawaybox}

\section{Theoretical Analysis of Early Exposure}
\label{sec:theory}

We find that mixing even a small fraction of $\dpost$ into pretraining substantially improves retention under subsequent fine-tuning, and that this effect is complementary to various post-training regularizations applied during post-training. This is striking, as there is little reason to expect that a small change to the pretraining corpus should persistently shape retention two stages downstream. Moreover, these benefits persist across a range of post-training interventions, suggesting that mixing operates on a different axis than techniques that regularize the post-training or fine-tuning stages. Broadly, our findings indicate that early exposure has a unique effect on how capabilities are implemented in the model, which then propagates through subsequent training stages.

We now formally characterize how early exposure affects the way post-trained capabilities are implemented in the model and their vulnerability to future forgetting. We analyze our three-stage pipeline in a two-layer linear model, which enables a precise characterization of feature learning during pretraining and its impact on subsequent training dynamics.

\textbf{Setup.} We train a two-layer linear network $\theta = \mathbf{W}_{1}\mathbf{W}_{2}$ sequentially on three regression tasks simulating $\dweb$, $\dpost$, and $\dft$, using gradient descent on the squared loss $\mathcal{L}(\theta;\mathcal{D}) = \mathbf{E}[||\theta\mathbf{x}-\mathbf{y}||^{2}]$. Each task is defined by an input distribution and a ground-truth linear map $\mathbf{A}^{t}$, with $\mathbf{y} = \mathbf{A}^{t}\mathbf{x}$. Following \citet{springer2025overtrainedlanguagemodelsharder}, we assume all tasks share singular vectors $\mathbf{U},\mathbf{V}$ so that $\mathbf{A}^{t} = \mathbf{U}\mathbf{\Sigma}_{t}\mathbf{V}^{\top}$; the singular values of $\mathbf{\Sigma}_{t}$ are the \emph{features} of task $t$.

\textbf{Feature structure.} We consider the input-space to be partitioned into three blocks, exhibiting different behaviors across the tasks:
\begin{itemize}[nosep,leftmargin=*]
\item \textit{Invariant features} ($n{-}2k$ features) have identical singular values across $\mathbf{A}^{\textrm{gen}}$, $\mathbf{A}^{\textrm{spec}}$, and $\mathbf{A}^{\textrm{ft}}$.
\item \textit{Inconsistent features} ($k$ features) have shared dimensions where the tasks disagree---the singular values differ between $\mathbf{A}^{\textrm{gen}}$, $\mathbf{A}^{\textrm{spec}}$, and $\mathbf{A}^{\textrm{ft}}$.
\item \textit{Specialized features} ($k$ features) are active only on $\dpost$, while having zero covariance under pretraining and downstream distributions.
\end{itemize}

\textbf{Task definitions.} The pretraining task $\dweb$ draws inputs from $x \sim \mathcal{N}(0, \mathbf{I}_{n-k})$, activating only shared dimensions. The post-training task $\dpost$ draws inputs from $x \sim \mathcal{N}(0, \mathbf{I}_{n})$, activating all dimensions including the specialized ones. Like pretraining, the downstream task $\dft$ draws inputs from $x \sim \mathcal{N}(0, \mathbf{I}_{n-k})$---critically, it does not activate the specialized features. We model \textit{early exposure} by training on the distribution $\mathcal{D}_{\textrm{mixed}} = (1{-}\alpha)\dweb + \alpha\dpost $. Full details and formal assumptions are given in Appendix~\ref{app:thy}.

\subsection{Early exposure learns different features}
We first characterize what features each pretraining strategy learns.
\begin{theorem}[\textit{Informal:} Only Early Exposure Learns Specialized Features] Let $\theta^{\textrm{mixed}}$ and $\theta^{\textrm{unmixed}}$ be the parameters learned by pretraining on $\mathcal{D}_{\textrm{mixed}}$ and $\dweb$, respectively, after sufficient training. Then $\theta^{\textrm{mixed}}$ learns the specialized features, while $\theta^{\textrm{unmixed}}$ does not.
\end{theorem}

The key mechanism is that linear networks learn features in descending order of their singular value~\citep{gidel2019implicitregularizationdiscretegradient,springer2025overtrainedlanguagemodelsharder}, with remaining features staying near zero. Without early exposure, the $k$ specialized features have the lowest singular values and are not learned. Early exposure, on the other hand, boosts the effective singular value above the learning threshold. See Appendix~\ref{app:thy} for the formal statement.

\textbf{Impact of a Small Mixing Ratio.} The mixing fraction $\alpha$ need not be large to have an impact. Suppose the specialized features have a singular value $\beta$ in the post-training task ($A^{\textrm{post}}$). Without exposure to the post-training data, specialized features have zero effective singular value and are thus never learned. Mixing an $\alpha$-fraction of $\dpost$ boosts the effective singular value to $\alpha\beta$, so when $\beta$ is sufficiently large, even small $\alpha$ suffices to cross the threshold at which specialized features are learned. Empirically, we also observe benefits from early exposure even when the quantity of mixed data is small relative to the total pretraining corpus.

\subsection{Post-training primarily reuses existing features}

Next, we study the impact of the different features learned by early exposure on the post-training process.

\begin{theorem}[\textit{Informal}: Post-training on $\theta^{\textrm{mixed}}$ versus $\theta^{\textrm{unmixed}}$] After sufficient post-training on $\dpost$, $\theta^{\textrm{mixed}}_{\textrm{post}}$ converges to parameters that leverage specialized features to reduce post-training loss, while $\theta^{\textrm{unmixed}}_{\textrm{post}}$ converges to parameters that only modify the inconsistent features.
\end{theorem}

The key insight is that features absent at initialization remain absent throughout post-training: if $\Sigma_{ii} = 0$ at the start of post-training, it stays $0$. Since $\theta^{\textrm{unmixed}}$ never learned the specialized features, post-training from this checkpoint can only reduce loss on $\dpost$ by distorting the inconsistent features. Post-training from $\theta^{\textrm{mixed}}$, having learned the specialized features during pretraining, can additionally improve loss on $\dpost$ using the specialized features. 

\colmtight\textbf{Cost of Specialization to General Performance.} Our analysis shows that, without early exposure, adaptation to $\dpost$ proceeds by modifying features that also support general capabilities—namely, the inconsistent features—thereby degrading performance on the pretraining task (i.e., increasing loss on $\dweb$). In contrast, early exposure enables the formation of specialized features that support post-training without interfering with other tasks. Consistent with the predictions of our analytical model, we observe that post-training following early exposure incurs less degradation in C4 loss than post-training from a base model without early exposure (Figure~\ref{fig:four_pipeline_frontiers}).

\subsection{Characterizing downstream forgetting behavior}

Finally, we show that the different feature usage established above directly determines forgetting under downstream fine-tuning.

\begin{theorem}[\textit{Informal:} $\theta^{\textrm{unmixed}}_{\textrm{post}}$ experiences more forgetting than $\theta_{\textrm{post}}^{\textrm{mixed}}$] Let $\Delta_{\textrm{mixed}}, \Delta_{\textrm{unmixed}}$ denote the increase in loss on $\dpost$ after fine-tuning on $\dft$, for the mixed and unmixed checkpoints, respectively. Then $\Delta_{\textrm{unmixed}} \geq \Delta_{\textrm{mixed}}$.
\end{theorem}

Because $\dft$ has no covariance along the specialized features, gradient updates during fine-tuning have zero projection along those directions. Inconsistent features, however, overlap with $\dft$ and are overwritten. Without early exposure, all the loss reduction on $\dpost$ is achieved by modifying the inconsistent features and is therefore vulnerable to erasure during fine-tuning. With early exposure, loss reduction is also implemented in isolated specialized features, enabling it to persist after further training. This provides a formal account of the frontier shift observed empirically: the mixed checkpoint's retention advantage traces back to feature learning during pretraining.

\colmtight\textbf{Summary.} Ultimately, our analysis shows that even limited exposure to post-training data during pretraining induces the formation of \textit{specialized features} for that domain. These features are isolated from both $\dgen$ and $\dft$, rendering them resistant to overwriting during fine-tuning. In contrast, without early exposure, post-training relies on modifying broadly shared features to reduce loss. This makes the adaptation inherently fragile and prone to forgetting under subsequent fine-tuning.
\section{Conclusion}

\colmtight\textbf{Main themes.} This work advances three claims. First, how well a capability survives later fine-tuning cannot be read off from how well a model performs on that capability immediately after it is acquired; two training recipes that look equivalent at the handoff can diverge substantially once the model is adapted downstream. Second, the \emph{manner} in which a capability is learned — when it is introduced during training, how it is presented, and what else the model is learning at the same time — shapes how durable that capability will be under later updates. Third, upstream training offers not a single knob but a family of interventions: early exposure during pretraining, replay during post-training, and regularization during post-training each shift the retention--adaptation frontier, and their effects are largely complementary rather than substitutable. Together, these observations reframe robustness to fine-tuning from a downstream problem to be mitigated into a design objective of upstream training.

\colmtight\textbf{Limitations.} Our experiments restrict the amount of post-training data introduced during pretraining to at most a single pass over that corpus; we do not characterize what happens when the post-training data is repeated much more aggressively — for instance, allocating several percent of the total pretraining budget to post-training data, which would imply many repetitions of a small corpus (see \cite{baek2026finetunersfallacypretrainfinetuning}). We focus on a single downstream adaptation method and do not characterize how early exposure interacts with preference-based or reinforcement learning fine-tuning schemes. We provide additional LoRA experiments in Appendix \ref{lora-appendix}.

\colmtight\textbf{Future work.} A fuller characterization of when early exposure helps — across varying degrees of overlap between upstream and downstream data, across a wider range of exposure levels, and into regimes of heavy repetition of post-training data — would sharpen practical guidance for upstream developers. Extending the study to larger models and longer training runs would test whether upstream interventions continue to shift the retention frontier at scale. A complementary line of work is algorithmic: whether new objectives or regularizers applied later in training can approximate the representational effect of early exposure, without modifying the pretraining corpus.

\newpage
\bibliography{colm2026_conference}

\begin{thebibliography}{29}
\providecommand{\natexlab}[1]{#1}
\providecommand{\url}[1]{\texttt{#1}}
\expandafter\ifx\csname urlstyle\endcsname\relax
  \providecommand{\doi}[1]{doi: #1}\else
  \providecommand{\doi}{doi: \begingroup \urlstyle{rm}\Url}\fi

\bibitem[Allal et~al.(2025)Allal, Lozhkov, Bakouch, Blázquez, Penedo,
  Tunstall, Marafioti, Kydlíček, Lajarín, Srivastav, Lochner, Fahlgren,
  Nguyen, Fourrier, Burtenshaw, Larcher, Zhao, Zakka, Morlon, Raffel, von
  Werra, and Wolf]{allal2025smollm2smolgoesbig}
Loubna~Ben Allal, Anton Lozhkov, Elie Bakouch, Gabriel~Martín Blázquez,
  Guilherme Penedo, Lewis Tunstall, Andrés Marafioti, Hynek Kydlíček,
  Agustín~Piqueres Lajarín, Vaibhav Srivastav, Joshua Lochner, Caleb
  Fahlgren, Xuan-Son Nguyen, Clémentine Fourrier, Ben Burtenshaw, Hugo
  Larcher, Haojun Zhao, Cyril Zakka, Mathieu Morlon, Colin Raffel, Leandro von
  Werra, and Thomas Wolf.
\newblock Smollm2: When smol goes big -- data-centric training of a small
  language model, 2025.
\newblock URL \url{https://arxiv.org/abs/2502.02737}.

\bibitem[Baek et~al.(2026)Baek, Monti, Schwab, Abbas, Adiga, Blakeney, Böther,
  Burstein, Carranza, Deng, Doshi, Dorna, Fang, Jiang, Joshi, Larsen, Lee,
  Mentzer, Merrick, Mongstad, Pan, Suri, Teh, Telanoff, Urbanek, Wang, Wills,
  Yin, Raghunathan, Kolter, Gaza, Morcos, Leavitt, and
  Maini]{baek2026finetunersfallacypretrainfinetuning}
Christina Baek, Ricardo~Pio Monti, David Schwab, Amro Abbas, Rishabh Adiga,
  Cody Blakeney, Maximilian Böther, Paul Burstein, Aldo~Gael Carranza, Alvin
  Deng, Parth Doshi, Vineeth Dorna, Alex Fang, Tony Jiang, Siddharth Joshi,
  Brett~W. Larsen, Jason~Chan Lee, Katherine~L. Mentzer, Luke Merrick, Haakon
  Mongstad, Fan Pan, Anshuman Suri, Darren Teh, Jason Telanoff, Jack Urbanek,
  Zhengping Wang, Josh Wills, Haoli Yin, Aditi Raghunathan, J.~Zico Kolter,
  Bogdan Gaza, Ari Morcos, Matthew Leavitt, and Pratyush Maini.
\newblock The finetuner's fallacy: When to pretrain with your finetuning data,
  2026.
\newblock URL \url{https://arxiv.org/abs/2603.16177}.

\bibitem[Bethune et~al.(2025)Bethune, Grangier, Busbridge, Gualdoni, Cuturi,
  and Ablin]{bethune2025scalinglawsforgettingfinetuning}
Louis Bethune, David Grangier, Dan Busbridge, Eleonora Gualdoni, Marco Cuturi,
  and Pierre Ablin.
\newblock Scaling laws for forgetting during finetuning with pretraining data
  injection, 2025.
\newblock URL \url{https://arxiv.org/abs/2502.06042}.

\bibitem[Biderman et~al.(2024)Biderman, Portes, Ortiz, Paul, Greengard,
  Jennings, King, Havens, Chiley, Frankle, Blakeney, and
  Cunningham]{biderman2024loralearnsforgets}
Dan Biderman, Jacob Portes, Jose Javier~Gonzalez Ortiz, Mansheej Paul, Philip
  Greengard, Connor Jennings, Daniel King, Sam Havens, Vitaliy Chiley, Jonathan
  Frankle, Cody Blakeney, and John~P. Cunningham.
\newblock Lora learns less and forgets less, 2024.
\newblock URL \url{https://arxiv.org/abs/2405.09673}.

\bibitem[Chen et~al.(2025)Chen, Huang, Gao, Wang, Yang, and
  Ji]{chen2025scalinglawspredictingdownstream}
Yangyi Chen, Binxuan Huang, Yifan Gao, Zhengyang Wang, Jingfeng Yang, and Heng
  Ji.
\newblock Scaling laws for predicting downstream performance in llms, 2025.
\newblock URL \url{https://arxiv.org/abs/2410.08527}.

\bibitem[Du et~al.(2024)Du, Zeng, Dong, and Tang]{duunderstanding}
Zhengxiao Du, Aohan Zeng, Yuxiao Dong, and Jie Tang.
\newblock Understanding emergent abilities of language models from the loss
  perspective.
\newblock In \emph{The Thirty-eighth Annual Conference on Neural Information
  Processing Systems}, 2024.

\bibitem[Gadre et~al.(2024)Gadre, Smyrnis, Shankar, Gururangan, Wortsman, Shao,
  Mercat, Fang, Li, Keh, Xin, Nezhurina, Vasiljevic, Jitsev, Soldaini, Dimakis,
  Ilharco, Koh, Song, Kollar, Carmon, Dave, Heckel, Muennighoff, and
  Schmidt]{gadre2024languagemodelsscalereliably}
Samir~Yitzhak Gadre, Georgios Smyrnis, Vaishaal Shankar, Suchin Gururangan,
  Mitchell Wortsman, Rulin Shao, Jean Mercat, Alex Fang, Jeffrey Li, Sedrick
  Keh, Rui Xin, Marianna Nezhurina, Igor Vasiljevic, Jenia Jitsev, Luca
  Soldaini, Alexandros~G. Dimakis, Gabriel Ilharco, Pang~Wei Koh, Shuran Song,
  Thomas Kollar, Yair Carmon, Achal Dave, Reinhard Heckel, Niklas Muennighoff,
  and Ludwig Schmidt.
\newblock Language models scale reliably with over-training and on downstream
  tasks, 2024.
\newblock URL \url{https://arxiv.org/abs/2403.08540}.

\bibitem[Gidel et~al.(2019)Gidel, Bach, and
  Lacoste-Julien]{gidel2019implicitregularizationdiscretegradient}
Gauthier Gidel, Francis Bach, and Simon Lacoste-Julien.
\newblock Implicit regularization of discrete gradient dynamics in linear
  neural networks, 2019.
\newblock URL \url{https://arxiv.org/abs/1904.13262}.

\bibitem[Hu et~al.(2021)Hu, Shen, Wallis, Allen-Zhu, Li, Wang, Wang, and
  Chen]{hu2021loralowrankadaptationlarge}
Edward~J. Hu, Yelong Shen, Phillip Wallis, Zeyuan Allen-Zhu, Yuanzhi Li, Shean
  Wang, Lu~Wang, and Weizhu Chen.
\newblock Lora: Low-rank adaptation of large language models, 2021.
\newblock URL \url{https://arxiv.org/abs/2106.09685}.

\bibitem[Kirkpatrick et~al.(2017)Kirkpatrick, Pascanu, Rabinowitz, Veness,
  Desjardins, Rusu, Milan, Quan, Ramalho, Grabska-Barwinska, Hassabis, Clopath,
  Kumaran, and Hadsell]{Kirkpatrick_2017}
James Kirkpatrick, Razvan Pascanu, Neil Rabinowitz, Joel Veness, Guillaume
  Desjardins, Andrei~A. Rusu, Kieran Milan, John Quan, Tiago Ramalho, Agnieszka
  Grabska-Barwinska, Demis Hassabis, Claudia Clopath, Dharshan Kumaran, and
  Raia Hadsell.
\newblock Overcoming catastrophic forgetting in neural networks.
\newblock \emph{Proceedings of the National Academy of Sciences}, 114\penalty0
  (13):\penalty0 3521–3526, March 2017.
\newblock ISSN 1091-6490.
\newblock \doi{10.1073/pnas.1611835114}.
\newblock URL \url{http://dx.doi.org/10.1073/pnas.1611835114}.

\bibitem[Kotha \& Liang(2026)Kotha and
  Liang]{kotha2026replayingpretrainingdataimproves}
Suhas Kotha and Percy Liang.
\newblock Replaying pre-training data improves fine-tuning, 2026.
\newblock URL \url{https://arxiv.org/abs/2603.04964}.

\bibitem[Lambert et~al.(2024)Lambert, Morrison, Pyatkin, Huang, Ivison,
  Brahman, Miranda, Liu, Dziri, Lyu, Gu, Malik, Graf, Hwang, Yang, Bras,
  Tafjord, Wilhelm, Soldaini, Smith, Wang, Dasigi, and
  Hajishirzi]{lambert2024tulu3}
Nathan Lambert, Jacob Morrison, Valentina Pyatkin, Shengyi Huang, Hamish
  Ivison, Faeze Brahman, Lester James~V. Miranda, Alisa Liu, Nouha Dziri, Shane
  Lyu, Yuling Gu, Saumya Malik, Victoria Graf, Jena~D. Hwang, Jiangjiang Yang,
  Ronan~Le Bras, Oyvind Tafjord, Chris Wilhelm, Luca Soldaini, Noah~A. Smith,
  Yizhong Wang, Pradeep Dasigi, and Hannaneh Hajishirzi.
\newblock T\"ulu 3: Pushing frontiers in open language model post-training,
  2024.
\newblock URL \url{https://arxiv.org/abs/2411.15124}.

\bibitem[Maini et~al.(2024)Maini, Feng, Schwarzschild, Lipton, and
  Kolter]{maini2024tofutaskfictitiousunlearning}
Pratyush Maini, Zhili Feng, Avi Schwarzschild, Zachary~C. Lipton, and J.~Zico
  Kolter.
\newblock Tofu: A task of fictitious unlearning for llms, 2024.
\newblock URL \url{https://arxiv.org/abs/2401.06121}.

\bibitem[Maini et~al.(2025)Maini, Goyal, Sam, Robey, Savani, Jiang, Zou,
  Fredrikson, Lipton, and Kolter]{maini2025safetypretraininggenerationsafe}
Pratyush Maini, Sachin Goyal, Dylan Sam, Alex Robey, Yash Savani, Yiding Jiang,
  Andy Zou, Matt Fredrikson, Zacharcy~C. Lipton, and J.~Zico Kolter.
\newblock Safety pretraining: Toward the next generation of safe ai, 2025.
\newblock URL \url{https://arxiv.org/abs/2504.16980}.

\bibitem[McCloskey \& Cohen(1989)McCloskey and
  Cohen]{mccloskey1989catastrophic}
Michael McCloskey and Neal~J Cohen.
\newblock Catastrophic interference in connectionist networks: The sequential
  learning problem.
\newblock In \emph{Psychology of learning and motivation}, volume~24, pp.\
  109--165. Elsevier, 1989.

\bibitem[Nishi et~al.(2025)Nishi, Ramesh, Okawa, Khona, Tanaka, and
  Lubana]{nishi2025representationshatteringtransformerssynthetic}
Kento Nishi, Rahul Ramesh, Maya Okawa, Mikail Khona, Hidenori Tanaka, and
  Ekdeep~Singh Lubana.
\newblock Representation shattering in transformers: A synthetic study with
  knowledge editing, 2025.
\newblock URL \url{https://arxiv.org/abs/2410.17194}.

\bibitem[O'Brien et~al.(2025)O'Brien, Casper, Anthony, Korbak, Kirk, Davies,
  Mishra, Irving, Gal, and
  Biderman]{obrien2025deepignorancefilteringpretraining}
Kyle O'Brien, Stephen Casper, Quentin Anthony, Tomek Korbak, Robert Kirk,
  Xander Davies, Ishan Mishra, Geoffrey Irving, Yarin Gal, and Stella Biderman.
\newblock Deep ignorance: Filtering pretraining data builds tamper-resistant
  safeguards into open-weight llms, 2025.
\newblock URL \url{https://arxiv.org/abs/2508.06601}.

\bibitem[Olmo et~al.(2025)Olmo, Ettinger, Bertsch, Kuehl, Graham, Heineman,
  Groeneveld, Brahman, Timbers, Ivison, Morrison, Poznanski, Lo, Soldaini,
  Jordan, Chen, Noukhovitch, Lambert, Walsh, Dasigi, Berry, Malik, Shah, Geng,
  Arora, Gupta, Anderson, Xiao, Murray, Romero, Graf, Asai, Bhagia, Wettig,
  Liu, Rangapur, Anastasiades, Huang, Schwenk, Trivedi, Magnusson, Lochner,
  Liu, Miranda, Sap, Morgan, Schmitz, Guerquin, Wilson, Huff, Bras, Xin, Shao,
  Skjonsberg, Shen, Li, Wilde, Pyatkin, Merrill, Chang, Gu, Zeng, Sabharwal,
  Zettlemoyer, Koh, Farhadi, Smith, and Hajishirzi]{olmo2025olmo3}
Team Olmo, Allyson Ettinger, Amanda Bertsch, Bailey Kuehl, David Graham, David
  Heineman, Dirk Groeneveld, Faeze Brahman, Finbarr Timbers, Hamish Ivison,
  Jacob Morrison, Jake Poznanski, Kyle Lo, Luca Soldaini, Matt Jordan, Mayee
  Chen, Michael Noukhovitch, Nathan Lambert, Pete Walsh, Pradeep Dasigi, Robert
  Berry, Saumya Malik, Saurabh Shah, Scott Geng, Shane Arora, Shashank Gupta,
  Taira Anderson, Teng Xiao, Tyler Murray, Tyler Romero, Victoria Graf, Akari
  Asai, Akshita Bhagia, Alexander Wettig, Alisa Liu, Aman Rangapur, Chloe
  Anastasiades, Costa Huang, Dustin Schwenk, Harsh Trivedi, Ian Magnusson,
  Jaron Lochner, Jiacheng Liu, Lester James~V. Miranda, Maarten Sap, Malia
  Morgan, Michael Schmitz, Michal Guerquin, Michael Wilson, Regan Huff,
  Ronan~Le Bras, Rui Xin, Rulin Shao, Sam Skjonsberg, Shannon~Zejiang Shen,
  Shuyue~Stella Li, Tucker Wilde, Valentina Pyatkin, Will Merrill, Yapei Chang,
  Yuling Gu, Zhiyuan Zeng, Ashish Sabharwal, Luke Zettlemoyer, Pang~Wei Koh,
  Ali Farhadi, Noah~A. Smith, and Hannaneh Hajishirzi.
\newblock Olmo 3, 2025.
\newblock URL \url{https://arxiv.org/abs/2512.13961}.

\bibitem[OLMo et~al.(2025)OLMo, Walsh, Soldaini, Groeneveld, Lo, Arora, Bhagia,
  Gu, Huang, Jordan, Lambert, Schwenk, Tafjord, Anderson, Atkinson, Brahman,
  Clark, Dasigi, Dziri, Ettinger, Guerquin, Heineman, Ivison, Koh, Liu, Malik,
  Merrill, Miranda, Morrison, Murray, Nam, Poznanski, Pyatkin, Rangapur,
  Schmitz, Skjonsberg, Wadden, Wilhelm, Wilson, Zettlemoyer, Farhadi, Smith,
  and Hajishirzi]{olmo2025olmo2}
Team OLMo, Pete Walsh, Luca Soldaini, Dirk Groeneveld, Kyle Lo, Shane Arora,
  Akshita Bhagia, Yuling Gu, Shengyi Huang, Matt Jordan, Nathan Lambert, Dustin
  Schwenk, Oyvind Tafjord, Taira Anderson, David Atkinson, Faeze Brahman,
  Christopher Clark, Pradeep Dasigi, Nouha Dziri, Allyson Ettinger, Michal
  Guerquin, David Heineman, Hamish Ivison, Pang~Wei Koh, Jiacheng Liu, Saumya
  Malik, William Merrill, Lester James~V. Miranda, Jacob Morrison, Tyler
  Murray, Crystal Nam, Jake Poznanski, Valentina Pyatkin, Aman Rangapur,
  Michael Schmitz, Sam Skjonsberg, David Wadden, Christopher Wilhelm, Michael
  Wilson, Luke Zettlemoyer, Ali Farhadi, Noah~A. Smith, and Hannaneh
  Hajishirzi.
\newblock 2 olmo 2 furious, 2025.
\newblock URL \url{https://arxiv.org/abs/2501.00656}.

\bibitem[Ouyang et~al.(2022)Ouyang, Wu, Jiang, Almeida, Wainwright, Mishkin,
  Zhang, Agarwal, Slama, Ray, Schulman, Hilton, Kelton, Miller, Simens, Askell,
  Welinder, Christiano, Leike, and
  Lowe]{ouyang2022traininglanguagemodelsfollow}
Long Ouyang, Jeff Wu, Xu~Jiang, Diogo Almeida, Carroll~L. Wainwright, Pamela
  Mishkin, Chong Zhang, Sandhini Agarwal, Katarina Slama, Alex Ray, John
  Schulman, Jacob Hilton, Fraser Kelton, Luke Miller, Maddie Simens, Amanda
  Askell, Peter Welinder, Paul Christiano, Jan Leike, and Ryan Lowe.
\newblock Training language models to follow instructions with human feedback,
  2022.
\newblock URL \url{https://arxiv.org/abs/2203.02155}.

\bibitem[Qi et~al.(2023)Qi, Zeng, Xie, Chen, Jia, Mittal, and
  Henderson]{qi2023finetuningalignedlanguagemodels}
Xiangyu Qi, Yi~Zeng, Tinghao Xie, Pin-Yu Chen, Ruoxi Jia, Prateek Mittal, and
  Peter Henderson.
\newblock Fine-tuning aligned language models compromises safety, even when
  users do not intend to!, 2023.
\newblock URL \url{https://arxiv.org/abs/2310.03693}.

\bibitem[Rolnick et~al.(2019)Rolnick, Ahuja, Schwarz, Lillicrap, and
  Wayne]{rolnick2019experiencereplaycontinuallearning}
David Rolnick, Arun Ahuja, Jonathan Schwarz, Timothy~P. Lillicrap, and Greg
  Wayne.
\newblock Experience replay for continual learning, 2019.
\newblock URL \url{https://arxiv.org/abs/1811.11682}.

\bibitem[Sam et~al.(2026)Sam, Goyal, Maini, Robey, and
  Kolter]{sam2026introducesafetyinterventionspretraining}
Dylan Sam, Sachin Goyal, Pratyush Maini, Alexander Robey, and J.~Zico Kolter.
\newblock When should we introduce safety interventions during pretraining?,
  2026.
\newblock URL \url{https://arxiv.org/abs/2601.07087}.

\bibitem[Springer et~al.(2025)Springer, Goyal, Wen, Kumar, Yue, Malladi,
  Neubig, and Raghunathan]{springer2025overtrainedlanguagemodelsharder}
Jacob~Mitchell Springer, Sachin Goyal, Kaiyue Wen, Tanishq Kumar, Xiang Yue,
  Sadhika Malladi, Graham Neubig, and Aditi Raghunathan.
\newblock Overtrained language models are harder to fine-tune, 2025.
\newblock URL \url{https://arxiv.org/abs/2503.19206}.

\bibitem[Srivastava et~al.(2014)Srivastava, Hinton, Krizhevsky, Sutskever, and
  Salakhutdinov]{10.5555/2627435.2670313}
Nitish Srivastava, Geoffrey Hinton, Alex Krizhevsky, Ilya Sutskever, and Ruslan
  Salakhutdinov.
\newblock Dropout: a simple way to prevent neural networks from overfitting.
\newblock \emph{J. Mach. Learn. Res.}, 15\penalty0 (1):\penalty0 1929–1958,
  January 2014.
\newblock ISSN 1532-4435.

\bibitem[Wortsman et~al.(2022{\natexlab{a}})Wortsman, Ilharco, Gadre, Roelofs,
  Gontijo-Lopes, Morcos, Namkoong, Farhadi, Carmon, Kornblith, and
  Schmidt]{wortsman2022modelsoupsaveragingweights}
Mitchell Wortsman, Gabriel Ilharco, Samir~Yitzhak Gadre, Rebecca Roelofs,
  Raphael Gontijo-Lopes, Ari~S. Morcos, Hongseok Namkoong, Ali Farhadi, Yair
  Carmon, Simon Kornblith, and Ludwig Schmidt.
\newblock Model soups: averaging weights of multiple fine-tuned models improves
  accuracy without increasing inference time, 2022{\natexlab{a}}.
\newblock URL \url{https://arxiv.org/abs/2203.05482}.

\bibitem[Wortsman et~al.(2022{\natexlab{b}})Wortsman, Ilharco, Kim, Li,
  Kornblith, Roelofs, Gontijo-Lopes, Hajishirzi, Farhadi, Namkoong, and
  Schmidt]{wortsman2022robustfinetuningzeroshotmodels}
Mitchell Wortsman, Gabriel Ilharco, Jong~Wook Kim, Mike Li, Simon Kornblith,
  Rebecca Roelofs, Raphael Gontijo-Lopes, Hannaneh Hajishirzi, Ali Farhadi,
  Hongseok Namkoong, and Ludwig Schmidt.
\newblock Robust fine-tuning of zero-shot models, 2022{\natexlab{b}}.
\newblock URL \url{https://arxiv.org/abs/2109.01903}.

\bibitem[Yang et~al.(2025)Yang, Li, Yang, Zhang, Hui, Zheng, Yu, Gao, Huang,
  Lv, Zheng, Liu, Zhou, Huang, Hu, Ge, Wei, Lin, Tang, Yang, Tu, Zhang, Yang,
  Yang, Zhou, Zhou, Lin, Dang, Bao, Yang, Yu, Deng, Li, Xue, Li, Zhang, Wang,
  Zhu, Men, Gao, Liu, Luo, Li, Tang, Yin, Ren, Wang, Zhang, Ren, Fan, Su,
  Zhang, Zhang, Wan, Liu, Wang, Cui, Zhang, Zhou, and
  Qiu]{yang2025qwen3technicalreport}
An~Yang, Anfeng Li, Baosong Yang, Beichen Zhang, Binyuan Hui, Bo~Zheng, Bowen
  Yu, Chang Gao, Chengen Huang, Chenxu Lv, Chujie Zheng, Dayiheng Liu, Fan
  Zhou, Fei Huang, Feng Hu, Hao Ge, Haoran Wei, Huan Lin, Jialong Tang, Jian
  Yang, Jianhong Tu, Jianwei Zhang, Jianxin Yang, Jiaxi Yang, Jing Zhou,
  Jingren Zhou, Junyang Lin, Kai Dang, Keqin Bao, Kexin Yang, Le~Yu, Lianghao
  Deng, Mei Li, Mingfeng Xue, Mingze Li, Pei Zhang, Peng Wang, Qin Zhu, Rui
  Men, Ruize Gao, Shixuan Liu, Shuang Luo, Tianhao Li, Tianyi Tang, Wenbiao
  Yin, Xingzhang Ren, Xinyu Wang, Xinyu Zhang, Xuancheng Ren, Yang Fan, Yang
  Su, Yichang Zhang, Yinger Zhang, Yu~Wan, Yuqiong Liu, Zekun Wang, Zeyu Cui,
  Zhenru Zhang, Zhipeng Zhou, and Zihan Qiu.
\newblock Qwen3 technical report, 2025.
\newblock URL \url{https://arxiv.org/abs/2505.09388}.

\bibitem[Yang et~al.(2023)Yang, Wang, Zhang, Petzold, Wang, Zhao, and
  Lin]{yang2023shadowalignmenteasesubverting}
Xianjun Yang, Xiao Wang, Qi~Zhang, Linda Petzold, William~Yang Wang, Xun Zhao,
  and Dahua Lin.
\newblock Shadow alignment: The ease of subverting safely-aligned language
  models, 2023.
\newblock URL \url{https://arxiv.org/abs/2310.02949}.

\end{thebibliography}
\bibliographystyle{colm2026_conference}

\newpage
\appendix
\raggedbottom

\section*{Contributions}

Lawrence Feng led the project and conducted all the main experiments.

Gaurav R. Ghosal seeded the idea of investigating the role of pretraining in catastrophic forgetting.

Gaurav R. Ghosal, Ziqian Zhong, Jacob Mitchell Springer, and Aditi Raghunathan were involved throughout the project, including shaping the project direction, experimental design, and paper writing.

\section*{Acknowledgments}

We gratefully acknowledge support from Apple, Google, Jane Street, the National Science Foundation and the FLAME cluster at Carnegie Mellon University.

This material is based upon work supported by the National Science Foundation Graduate Research Fellowship under Grant No. DGE2140739. Any opinions, findings, conclusions, or recommendations expressed in this material are those of the authors and do not necessarily reflect the views of the National Science Foundation.

We'd like to thank Christina Baek for her insights on the dynamics of pretraining and fine-tuning. We'd like to thank Kevin Li for his feedback on earlier versions of this work.

\ifarxiv\else

\fi

\newpage 

\section{Training Details}
\label{appendix:training}
\FloatBarrier

\subsection{SmolLM2-1B model architecture}
\label{app:model_1b}

\begin{table}[h]
\centering
\caption{SmolLM2-1B model architecture (custom config interpolated from SmolLM2 family).}
\label{tab:model_arch_1b}
\begin{tabular}{ll}
\toprule
\textbf{Parameter} & \textbf{Value} \\
\midrule
Parameters & 1.03B \\
Hidden dimension & 1,728 \\
Attention heads & 27 \\
Layers & 24 \\
Head dimension & 64 \\
Query groups & 27 (MHA) \\
MLP intermediate size & 4,608 \\
Vocabulary size & 49,152 \\
Context length & 8,192 (max), 1,024 (training) \\
Normalization & RMSNorm \\
Position encoding & RoPE (base=100,000) \\
\bottomrule
\end{tabular}
\end{table}

\subsection{Dataset statistics}
\label{app:datasets}

\begin{table}[hbt!]
\centering
\caption{135M Parameter Experiments}
\label{tab:datasets}
\begin{tabular}{llrl}
\toprule
\textbf{Dataset} & \textbf{Split} & \textbf{Tokens} & \textbf{Description} \\
\midrule
C4 & Train & 8.7B & General web text pretraining corpus \\
MusicPile & Train & 0.3B & Music-domain text corpus \\
ChemPile & Train & 0.3B & Chemistry-domain text corpus \\
FLAN & Train & 0.3B & Instruction-tuning dataset \\
\bottomrule
\end{tabular}
\end{table}

\begin{table}[hbt!]
\centering
\caption{1B Parameter Experiments}
\label{tab:datasets_1b}
\begin{tabular}{llrl}
\toprule
\textbf{Dataset} & \textbf{Split} & \textbf{Tokens} & \textbf{Description} \\
\midrule
C4 & Train & 19.7B & General web text pretraining corpus \\
MusicPile & Train & 0.3B & Music-domain text corpus \\
ChemPile & Train & 0.3B & Chemistry-domain text corpus \\
FLAN & Train & 0.3B & Instruction-tuning dataset \\
\bottomrule
\end{tabular}
\end{table}

\subsection{Optimizer configuration}
\label{app:optimizer}

\begin{table}[hbt!]
\centering
\caption{Optimizer configuration used across all experiments.}
\label{tab:optimizer}
\begin{tabular}{ll}
\toprule
\textbf{Parameter} & \textbf{Value} \\
\midrule
Optimizer & AdamW \\
$\beta_1$ & 0.9 \\
$\beta_2$ & 0.95 \\
Gradient clipping & 1.0 (max norm) \\
Precision & bf16-mixed \\
\bottomrule
\end{tabular}
\end{table}

\newpage

\subsection{1B Stage 1 pretraining}
\label{app:stage1_1b}

\begin{table}[h]
\centering
\caption{Stage 1 pretraining configuration for 1B experiments.}
\label{tab:stage1_1b}
\begin{tabular}{lll}
\toprule
\textbf{Parameter} & \textbf{Value} & \textbf{Notes} \\
\midrule
Total tokens & 20B & Chinchilla-optimal for 1.03B params \\
C4 corpus & 21.0B tokens & 120 shards, tokenized \\
Learning rate & 5e-4 & Peak LR \\
Minimum LR & 5e-5 & Cosine decay target \\
Warmup steps & 1,000 & \\
Global batch size & 512 & \\
Micro batch size & 30 & Per-GPU \\
Eval interval & 1,000 steps & \\
Save interval & 1,000 steps & \\
GPUs & 8$\times$ L40S & Single node \\
Seed & 42 & \\
\bottomrule
\end{tabular}
\end{table}

\subsection{Post-training hyperparameters}

\begin{table}[hbt!]
\centering
\caption{FFT hyperparameter search space for Stage 2 post-training (135M).}
\label{tab:fft_config}
\begin{tabular}{lll}
\toprule
\textbf{Parameter} & \textbf{Values} & \textbf{Notes} \\
\midrule
Learning rate & \{1e-4, 2e-4, 5e-4, 1e-3, 5e-3\} & Peak LR \\
Minimum LR & 5e-5 & Cosine decay target \\
Dropout & \{0.0, 0.02, 0.05\} & embed/attn/resid/mlp \\
Weight decay & 0.1 & Fixed \\
Warmup steps & 500 & Fixed \\
Batch size & \{192, 480, 896\} & Global batch size \\
Max tokens & 2B & With early stopping \\
\bottomrule
\end{tabular}
\end{table}

\begin{table}[h]
\centering
\caption{Stage 2 FFT hyperparameter search space for 1B-scale post-training.}
\label{tab:stage2_1b_fft}
\begin{tabular}{lll}
\toprule
\textbf{Parameter} & \textbf{Values} & \textbf{Notes} \\
\midrule
Learning rate & \{1e-5, 2e-5, 5e-5, 1e-4, 2e-4\} & Peak LR \\
Minimum LR & 5e-5 & Cosine decay target \\
Dropout & 0.0 & Fixed (see \S\ref{app:dropout_1b} for ablation) \\
Weight decay & 0.1 & Fixed \\
Warmup steps & 100 & \\
Global batch size & 512 & \\
Micro batch size & 30 & Per-GPU \\
Max tokens & 2B & With early stopping \\
CPT budget & 300M & Tokens per epoch (subsampled) \\
Early stopping patience & 3 & Evaluation intervals \\
Evaluation interval & 100 steps & \\
GPUs & 8$\times$ L40S & Single node \\
Seed & 40 & \\
\bottomrule
\end{tabular}
\end{table}

\newpage

\subsection{LoRA configuration}

\begin{table}[hbt!]
\centering
\caption{LoRA hyperparameter configuration for Stage 2 post-training.}
\label{tab:lora_config}
\begin{tabular}{lll}
\toprule
\textbf{Parameter} & \textbf{Values} & \textbf{Notes} \\
\midrule
LoRA rank ($r$) & 64 & Fixed \\
LoRA alpha ($\alpha$) & 128 & Fixed, $\alpha/r = 2$ \\
LoRA dropout & \{0.0, 0.02, 0.05\} & Same as FFT dropout \\
LoRA targets & projection, mlp, head & Q/K/V excluded \\
Learning rate & \{1e-4, 2e-4, 5e-4, 1e-3, 5e-3\} & Same as FFT \\
Weight decay & 0.1 & Same as FFT \\
Other parameters & \multicolumn{2}{l}{Same as FFT (Table~\ref{tab:fft_config})} \\
\bottomrule
\end{tabular}
\end{table}

\subsection{1B Stage 2 CPT: dropout ablation}
\label{app:dropout_1b}

\begin{table}[h]
\centering
\caption{Dropout ablation configuration for 1B Stage 2 CPT. MusicPile CPT pipeline only. 12 runs total (2 $\lambda$ $\times$ 3 LRs $\times$ 2 dropout rates).}
\label{tab:dropout_1b}
\begin{tabular}{lll}
\toprule
\textbf{Parameter} & \textbf{Values} & \textbf{Notes} \\
\midrule
Dropout & \{0.02, 0.05\} & embed/attn/resid/mlp \\
Learning rate & \{1e-5, 2e-5, 5e-5\} & Best 3 from baseline \\
$\lambda$ & \{0.0, 1.0\} & MusicPile mixing only \\
CPT dataset & MusicPile & Priority pipeline \\
Other parameters & \multicolumn{2}{l}{Same as baseline (Table~\ref{tab:stage2_1b_fft})} \\
\bottomrule
\end{tabular}
\end{table}

\newpage

\section{Additional Plots}

\subsection{135M dropout and replay frontiers with retained pretraining loss (C4)}

\begin{figure}[H]
    \centering
    \includegraphics[width=\textwidth]{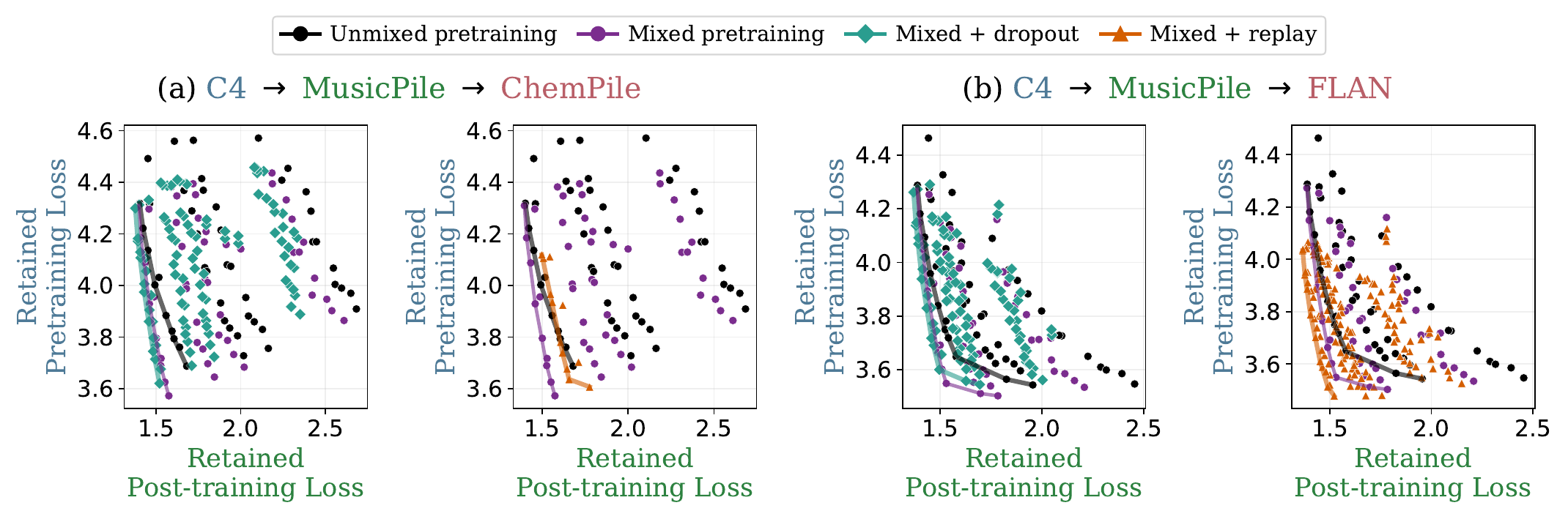}
    \caption{\textbf{Dropout and replay preserve broader pretraining capability in addition to the post-training capability (135M).} Companion to Figure~\ref{fig:mixing_combined}, plotting retained pretraining performance against retained post-training performance. Adding dropout or replay to mixed pretraining further lowers retained pretraining loss at matched retained post-training loss.}
    \label{fig:135M-dropout-replay-pt-c4}
\end{figure}

\subsection{Additional 135M dropout and replay frontiers (dropout and replay without mixing)}

\begin{figure}[H]
    \centering
    \includegraphics[width=\textwidth]{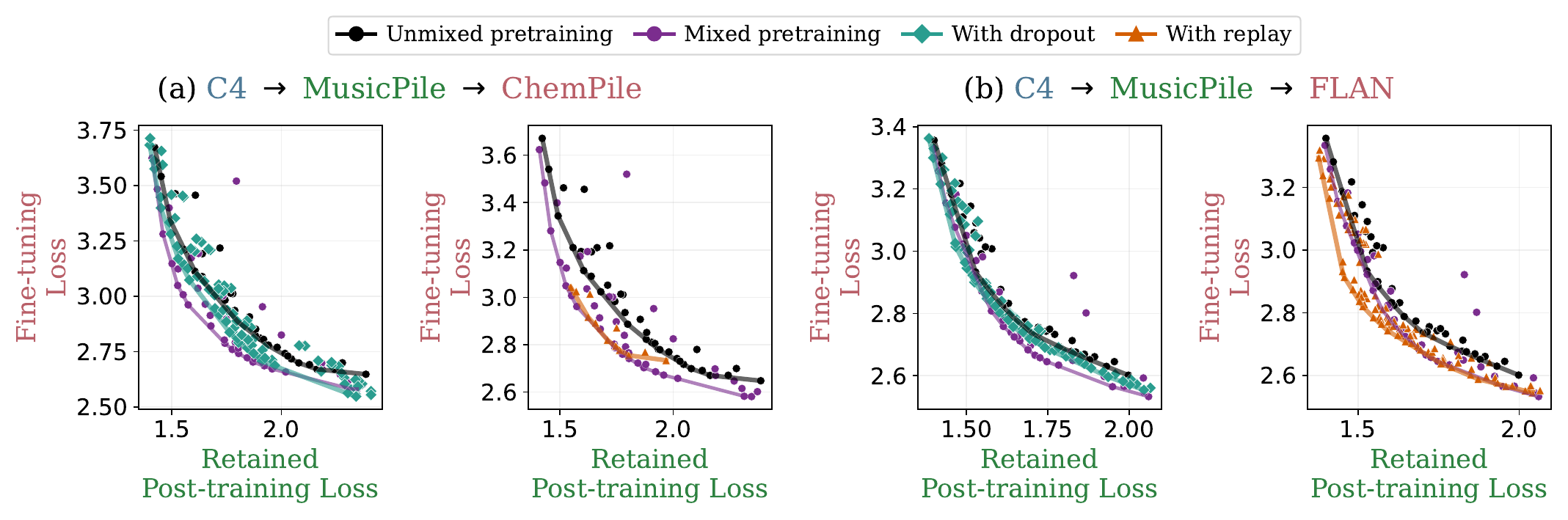}
    \caption{\textbf{Dropout and replay applied without mixing (135M).} To isolate the effect of post-training interventions from pretraining-time mixing, each panel applies dropout or replay on top of \emph{unmixed} pretraining ($\lambda=0$), with the mixed-pretraining frontier shown for reference. Within each pipeline, the left panel adds dropout during Stage~2 post-training; the right panel adds a small fraction ($1\%$) of general-domain replay. Both interventions shift the fine-tuning--retention frontier toward the lower left.}
    \label{fig:135M-dropout-replay-unmixed}
\end{figure}

\subsection{135M LoRA experiments}
\label{lora-appendix}

\begin{figure}[H]
    \centering
    \includegraphics[width=\textwidth]{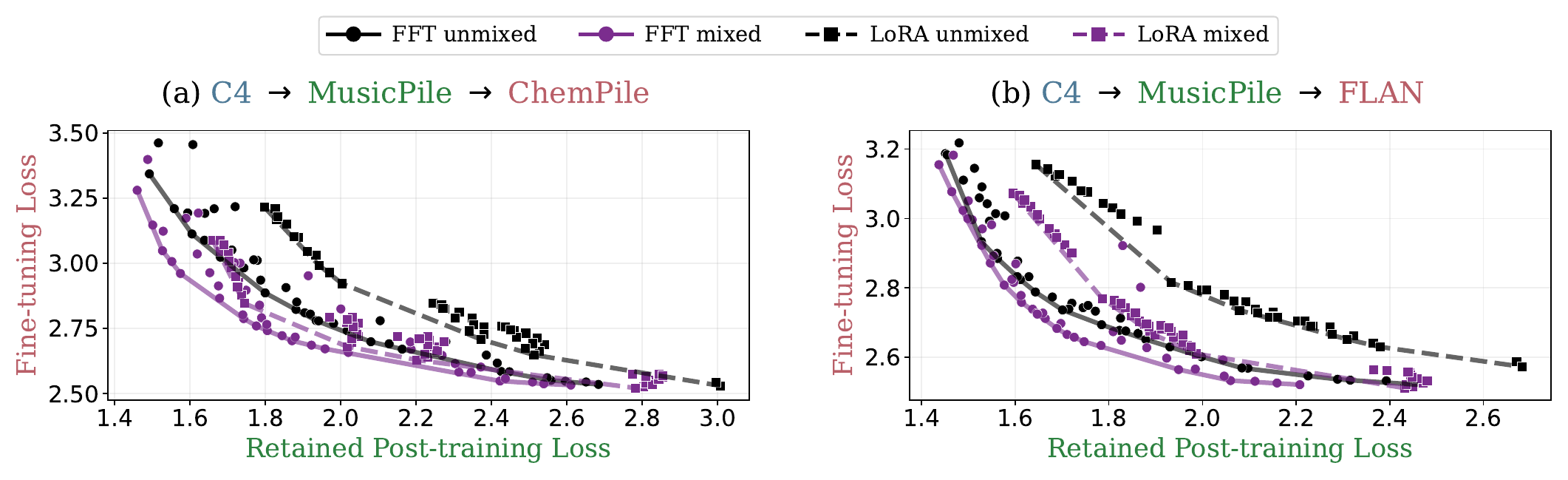}
    \caption{\textbf{FFT vs LoRA fine-tuning--retention frontiers (135M).} Each panel shows four frontiers obtained by sweeping Stage~2 post-training hyperparameters and Stage~3 fine-tuning learning rates. Mixed pretraining improves both the FFT and LoRA frontiers in both downstream settings, and FFT generally attains a better fine-tuning--retention tradeoff than LoRA at matched upstream configurations.}
    \label{fig:135M-fft-lora}
\end{figure}

\newpage
\subsection{1B experiments}
\label{app:1b-plots}

\textbf{Black} denotes the frontier obtained from unmixed pretraining, and \textbf{\color{violet}purple} denotes the frontier obtained from mixed pretraining.

\subsubsection{Mixing frontiers}

\begin{figure}[H]
    \centering
    \includegraphics[width=\textwidth]{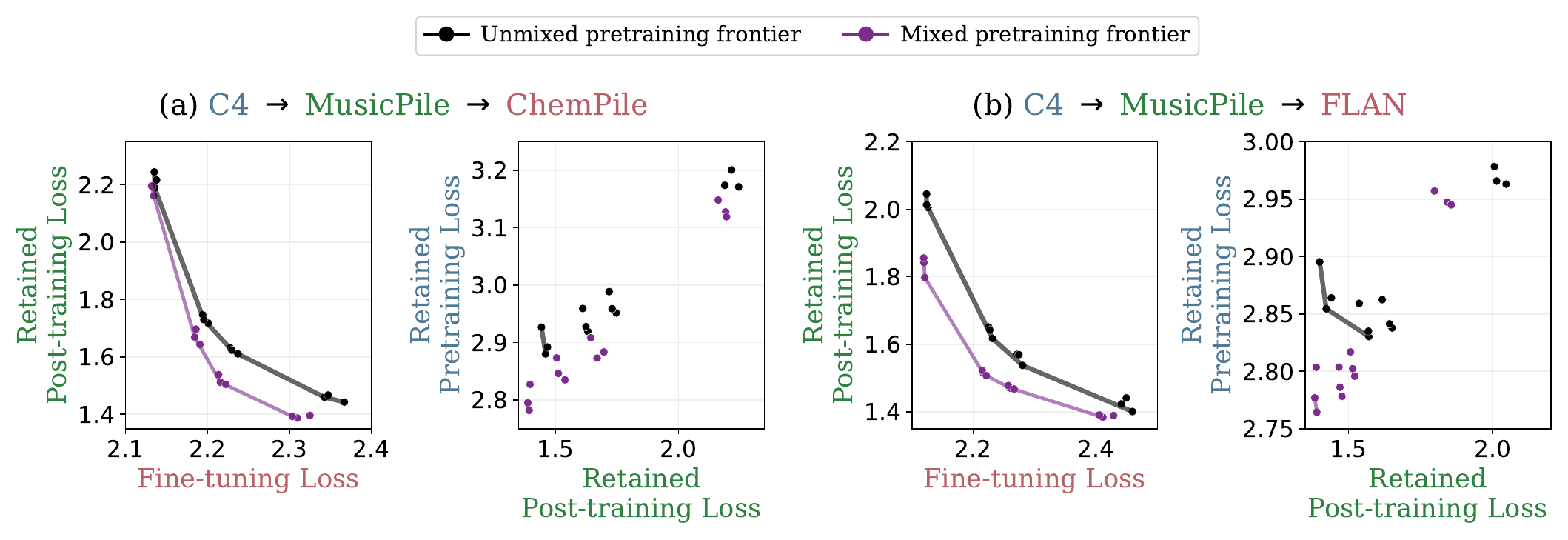}
    \caption{\textbf{Mixing frontiers at 1B, MusicPile post-training pipelines.} Companion to Figure~\ref{fig:four_pipeline_frontiers} at 1B scale. Within each pipeline, the left panel plots retained post-training loss against downstream fine-tuning loss, and the right panel plots retained pretraining loss against retained post-training loss. As at the 135M scale, mixed pretraining consistently shifts the frontier toward lower retained post-training loss, lower retained pretraining loss, and lower downstream fine-tuning loss.}
    \label{fig:1b-mixed-mp}
\end{figure}

\begin{figure}[H]
    \centering
    \includegraphics[width=0.5\textwidth]{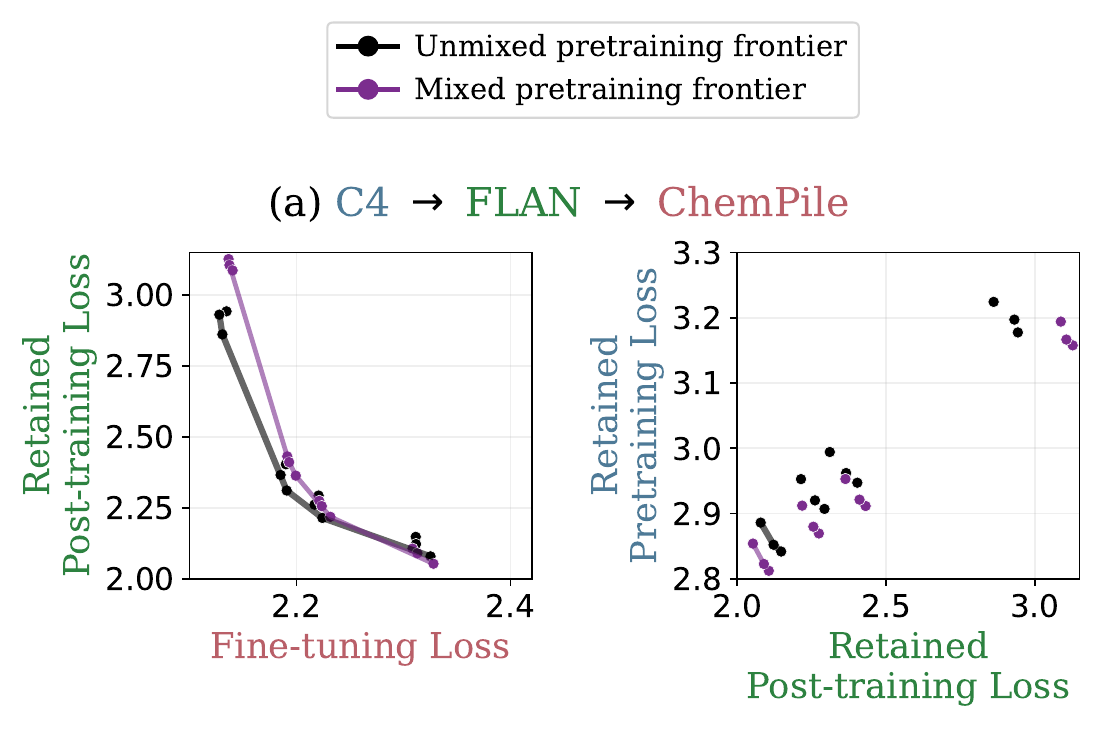}
    \caption{\textbf{Mixing frontiers at 1B, FLAN post-training pipeline.} Left: retained post-training loss vs fine-tuning loss. Right: retained pretraining loss (C4) vs retained post-training loss. In this pipeline, mixed pretraining does not noticeably improve the retained post-training vs fine-tuning frontier (left), but it does improve the retained pretraining vs retained post-training frontier (right).}
    \label{fig:1b-mixed-flan-chempile}
\end{figure}

\newpage
\begin{figure}[H]
    \centering
    \includegraphics[width=0.45\textwidth]{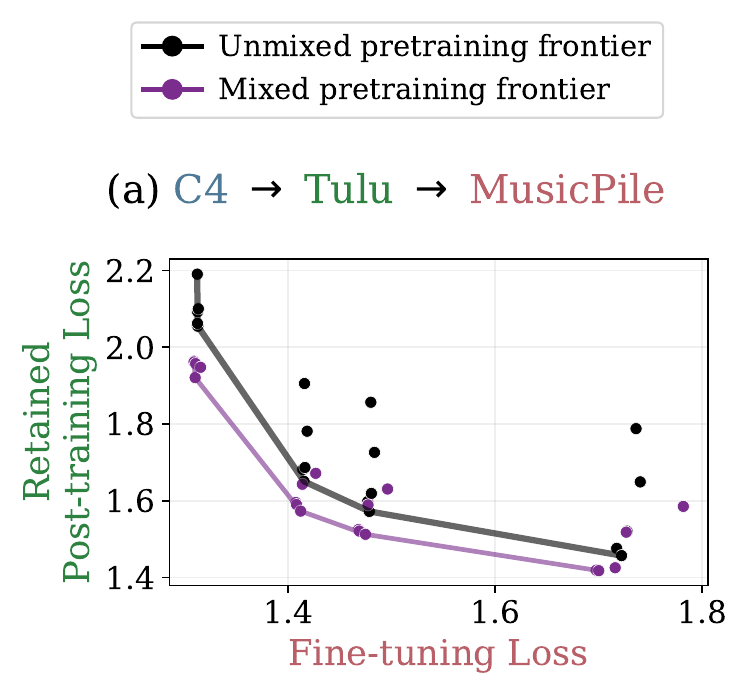}
    \caption{\textbf{Mixing frontier at 1B with a modern instruction-tuning mixture (Tulu-3-SFT).} As an additional check that the benefit of early exposure is not specific to single-domain post-training corpora, we repeat the 1B frontier sweep with Tulu-3-SFT \citep{lambert2024tulu3}, a modern instruction-tuning mixture spanning chat, instruction-following, math, code, safety, and multilingual data, as the post-training corpus (C4 $\rightarrow$ Tulu-3-SFT $\rightarrow$ MusicPile). The mixed-pretraining frontier again dominates the unmixed frontier.}
    \label{fig:1b-mixed-tulu}
\end{figure}

\newpage
\subsubsection{Upstream dropout and replay at 1B parameters}
\label{app:1b-dropout-replay}

At 1B we sweep replay at two fractions, 1\% and 10\% of the post-training budget, since the optimal replay fraction varies with the domain and training scale \citep{bethune2025scalinglawsforgettingfinetuning, kotha2026replayingpretrainingdataimproves}. The figures in this section pool both fractions; restricting to either alone yields nearly identical frontiers, so our conclusions are not sensitive to this choice.

\begin{figure}[H]
    \centering
    \includegraphics[width=\textwidth]{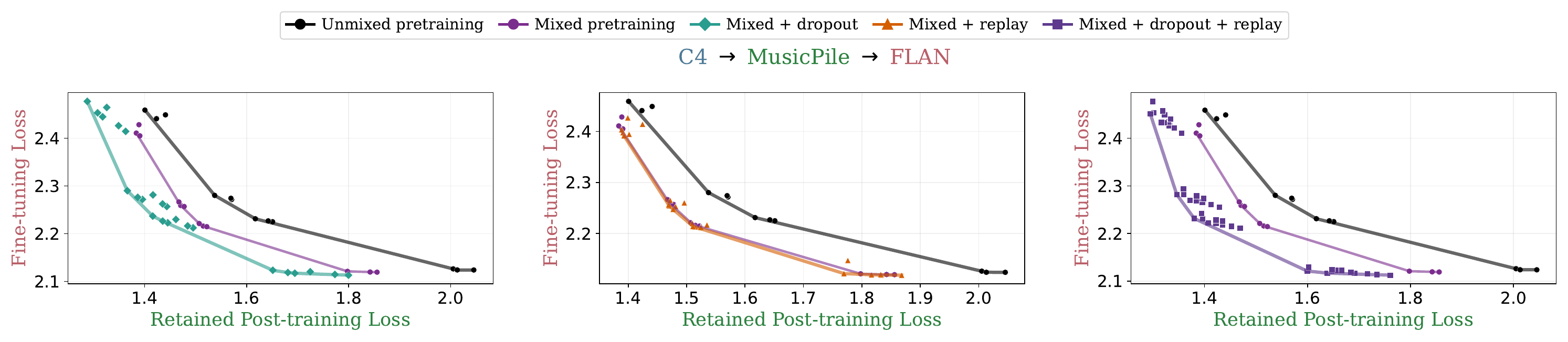}
    \caption{\textbf{Dropout and replay on top of mixed pretraining at 1B.} Companion to Figure~\ref{fig:mixing_combined} at 1B scale. The panels add Stage~2 dropout (left), replay (middle), and their combination (right) on top of mixed pretraining, with the unmixed and mixed frontiers repeated in every panel for reference. As at 135M, dropout further shifts the fine-tuning--retention frontier beyond mixed pretraining alone. Replay alone, in contrast, tracks the mixed-pretraining frontier closely at this scale, consistent with \citet{kotha2026replayingpretrainingdataimproves}, who find that replay becomes unimportant once the target data is partially seen during pretraining. Combining dropout and replay (right) yields the best frontier overall, dominating all other configurations.}
    \label{fig:1b-interventions-mixed}
\end{figure}

\begin{figure}[H]
    \centering
    \includegraphics[width=\textwidth]{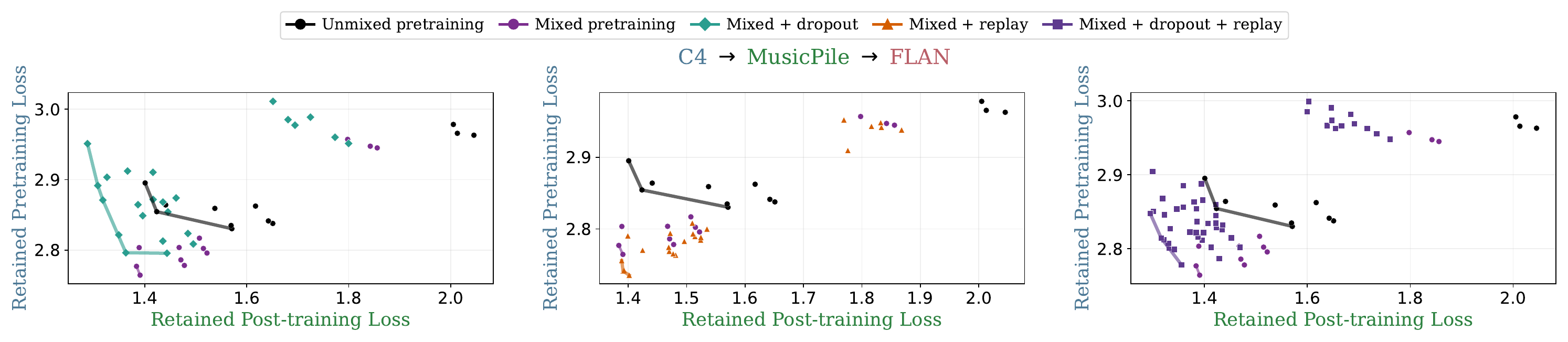}
    \caption{\textbf{Dropout and replay on top of mixed pretraining, broader pretraining retention (1B).} Companion to Figure~\ref{fig:135M-dropout-replay-pt-c4} at 1B scale: the same Stage~2 sweeps plotted against retained pretraining loss on C4. Replay (middle) preserves and often slightly improves C4 retention relative to mixed pretraining alone, consistent with its role of keeping general-domain data present during Stage~2. Dropout (left), at the rates we swept, degrades C4 retention for many configurations, which we attribute to insufficient tuning of the dropout rate; the combined intervention (right) inherits this degradation from its dropout component.}
    \label{fig:1b-interventions-mixed-c4}
\end{figure}

\newpage

\begin{figure}[H]
    \centering
    \includegraphics[width=\textwidth]{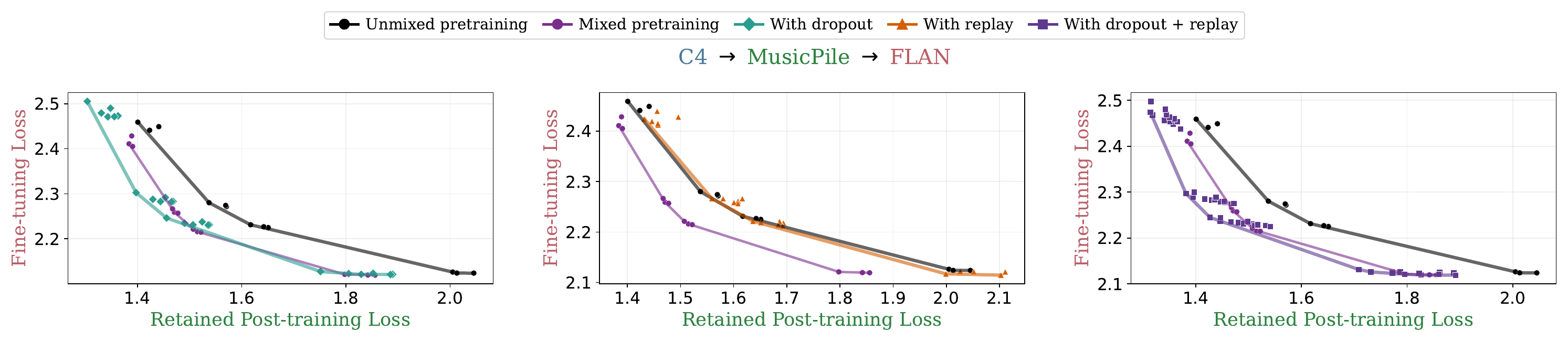}
    \caption{\textbf{Dropout and replay applied without pretraining-time mixing (1B).} Companion to Figure~\ref{fig:135M-dropout-replay-unmixed} at 1B scale. Each panel applies dropout (left), replay (middle), or both (right) on top of unmixed pretraining ($\lambda=0$), with the mixed-pretraining frontier shown for reference. As at 135M, dropout shifts the fine-tuning--retention frontier well beyond the unmixed baseline, while replay alone has little effect at this scale; combining the two (right) behaves similarly to dropout alone.}
    \label{fig:1b-interventions-unmixed}
\end{figure}

\newpage
\subsubsection{Sensitivity to the Stage 3 fine-tuning budget}
\label{app:1b-stage3-evolution}

\begin{figure}[H]
    \centering
    \includegraphics[width=\textwidth]{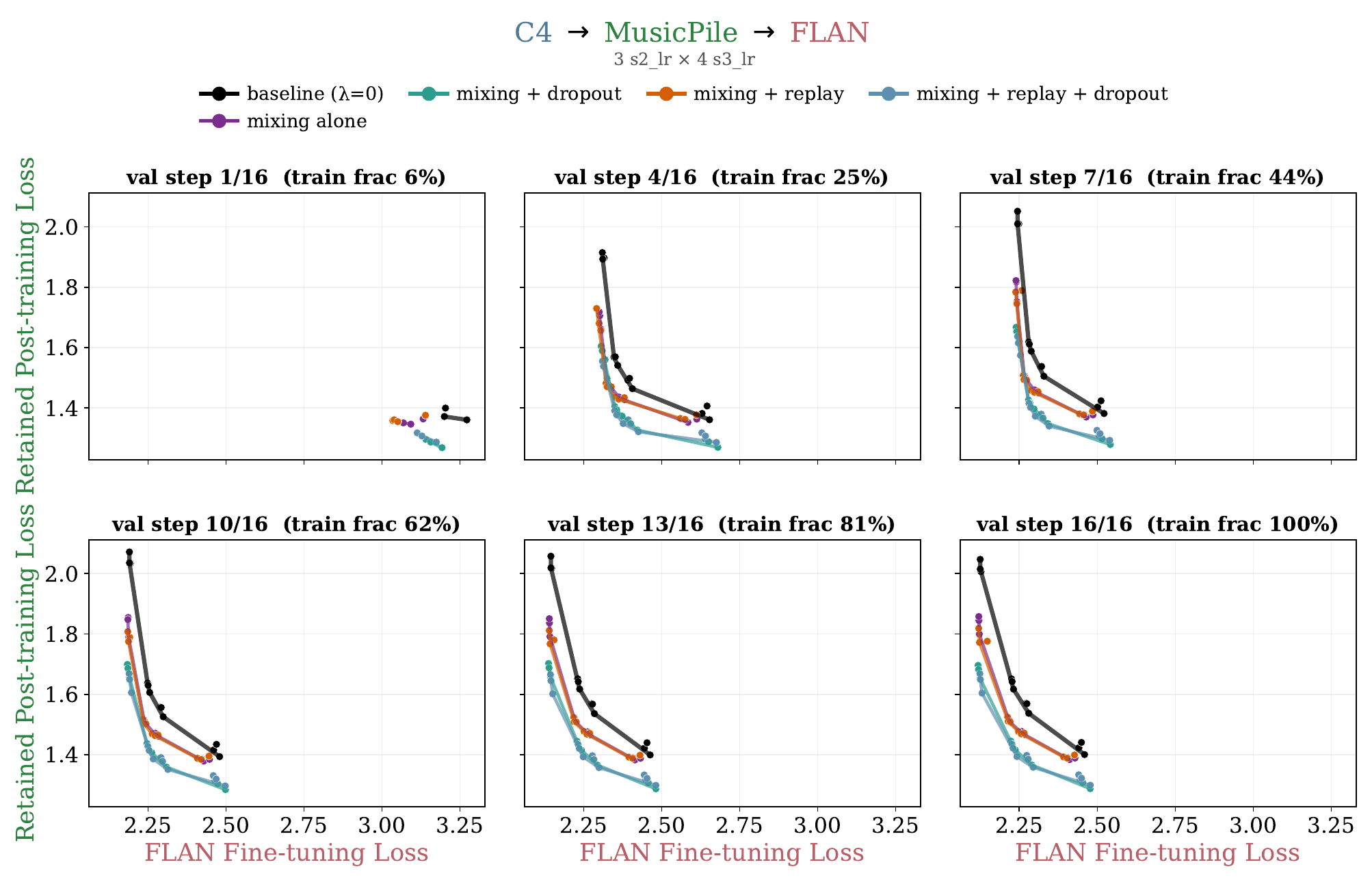}
    \caption{\textbf{Frontier ordering is established early in Stage~3 fine-tuning and preserved throughout (1B).} Each panel shows the fine-tuning--retention frontiers of the C4 $\rightarrow$ MusicPile $\rightarrow$ FLAN pipeline at a fixed fraction of the Stage~3 fine-tuning budget (from 6\% to 100\%), for the unmixed baseline, mixing alone, mixing + dropout, mixing + replay, and mixing + replay + dropout (3 Stage~2 learning rates $\times$ 4 Stage~3 learning rates per intervention). The relative ordering of the intervention frontiers emerges very early in fine-tuning and is preserved through to the final checkpoint, indicating that our frontier comparisons are not sensitive to the exact Stage~3 fine-tuning budget.}
    \label{fig:1b-stage3-evolution}
\end{figure}

\newpage
\subsection{Additional Plots on the Latent Effect of Early Exposure}
  \label{app:latent-effective-mixing}

  \begin{figure}[H]
      \centering
      \begin{subfigure}[t]{0.48\textwidth}
          \centering
          \includegraphics[width=\linewidth]{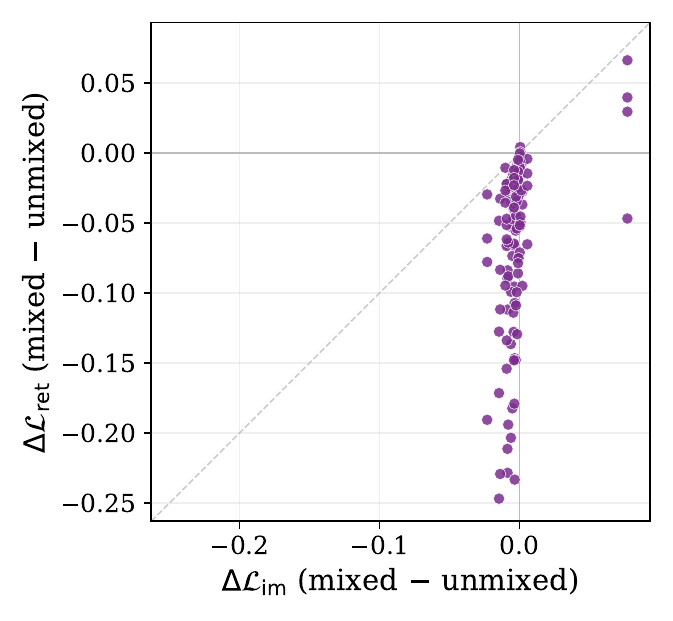}
          \caption{135M parameters.}
          \label{fig:delta-mixing-135m}
      \end{subfigure}
      \hfill
      \begin{subfigure}[t]{0.48\textwidth}
          \centering
          \includegraphics[width=\linewidth]{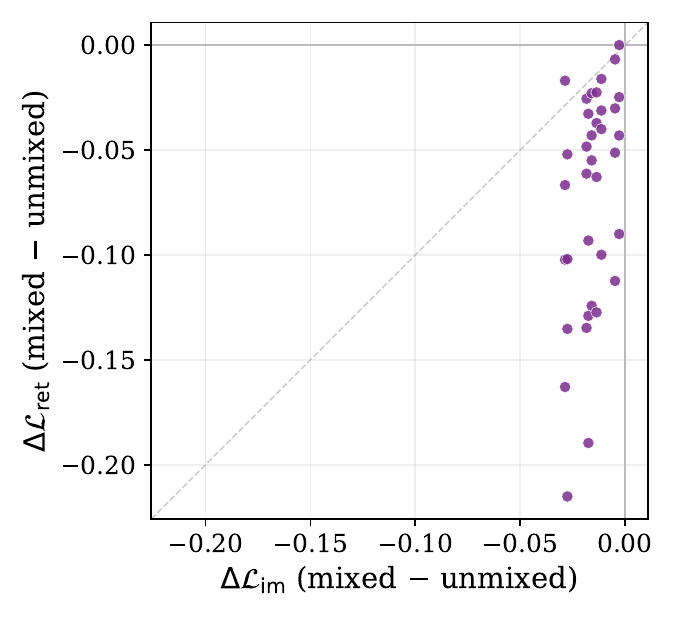}
          \caption{1B parameters.}
          \label{fig:delta-mixing-1b}
      \end{subfigure}
      \caption{\textbf{The benefits of early exposure cannot be predicted from post-training performance.} We sweep training hyperparameters across the C4 $\rightarrow$ MusicPile $\rightarrow$ FLAN pipeline, with and without early exposure, forming pairs where the only difference is whether or not the model had early exposure to the post-training domain. We measure the immediate $\limm$ and retained $\lret$ performance on the post-training domain. We compute the difference in $\limm$ and $\lret$ between pairs of models. We see that points cluster vertically near $\Delta \limm \approx 0$ with substantial negative $\Delta \lret$. That is, models trained with and without early exposure perform similarly after Stage 2, but differ substantially after Stage 3, often performing significantly better (i.e. lower loss).}
      \label{fig:delta-mixing}
  \end{figure}

\newpage
\subsection{Base-Model Familiarity Predicts Robustness to Forgetting in OLMo-2-7B}
\label{app:olmo2}

Our main experiments establish the effect of early exposure causally, by varying the fraction of post-training data mixed into pretraining. Here we ask whether the same mechanism is visible in an off-the-shelf model at larger scale, using the public OLMo-2-1124-7B base and SFT checkpoints \citep{olmo2025olmo2}. On a frozen public model we cannot manipulate exposure directly, so we use per-example base-model loss as a proxy: low base-model loss indicates content the model was already familiar with after pretraining, which we predict should be forgotten less. Concretely, we fine-tune the SFT checkpoint on ChemPile (52M tokens) as the forgetting stage, and score per-example response-token loss on 50{,}000 examples of Tulu-3-SFT \citep{lambert2024tulu3}, the SFT training corpus, before and after this fine-tuning.

Consistent with our main results, forgetting rises monotonically with base-model loss (Figure~\ref{fig:olmo2}a): SFT examples the base model already finds familiar undergo less forgetting under subsequent fine-tuning. To confirm this reflects base-model familiarity rather than what the SFT model is currently worse at, we control for SFT-model loss: within a narrow slice of SFT-model loss, the relationship persists (Figure~\ref{fig:olmo2}b), and comparing familiar against unfamiliar examples matched on SFT-model loss, the familiar group forgets roughly $4\times$ less. The positive slope also persists within individual source domains of the Tulu-3 mixture, ruling out differences between source domains as the driver. Together, these results show that base-model familiarity predicts robustness to forgetting at 7B scale, consistent with the mechanism behind pretraining-time mixing established causally in the paper.

\begin{figure}[H]
    \centering
    \begin{subfigure}[t]{0.48\textwidth}
        \centering
        \includegraphics[width=\linewidth]{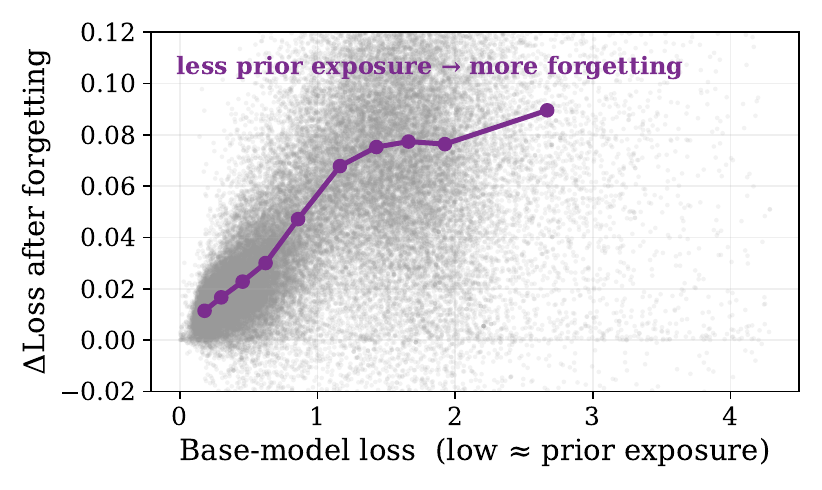}
        \caption{All 50k Tulu-3-SFT examples.}
        \label{fig:olmo2-base-vs-delta}
    \end{subfigure}
    \hfill
    \begin{subfigure}[t]{0.48\textwidth}
        \centering
        \includegraphics[width=\linewidth]{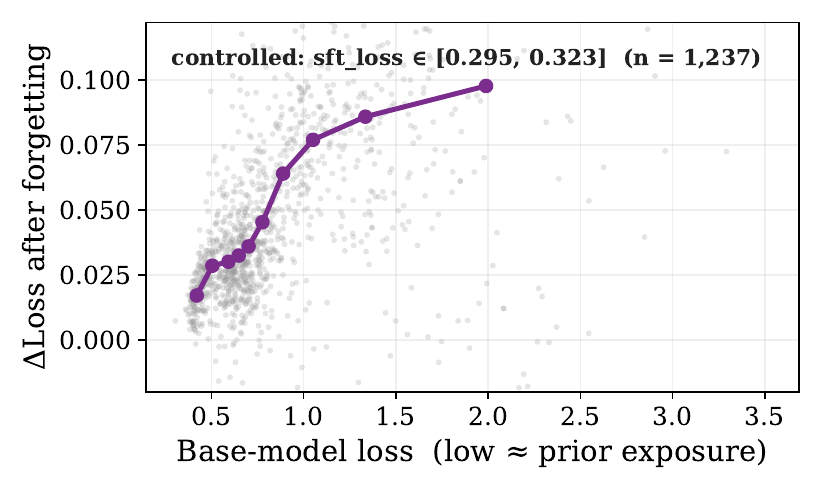}
        \caption{Controlling for SFT-model loss.}
        \label{fig:olmo2-slice}
    \end{subfigure}
    \caption{\textbf{Content the base model already knows is forgotten less (OLMo-2-7B).} Per-example change in loss after fine-tuning the OLMo-2-7B-SFT checkpoint on ChemPile, plotted against base-model loss (gray: individual Tulu-3-SFT examples; purple: bin means). \textbf{(a)} Forgetting rises monotonically with base-model loss: content already familiar after pretraining (low base-model loss) is more robust to subsequent fine-tuning. \textbf{(b)} The same relationship restricted to a narrow slice of SFT-model loss ($n=1{,}237$), showing the trend is not explained by what the SFT model is currently worse at; other slices look qualitatively identical.}
    \label{fig:olmo2}
\end{figure}

\newpage

\newcommand{\xmixed}[0]{\mathbf{X}^{\textrm{(mixed)}}}
\newcommand{\ymixed}[0]{\mathbf{Y}^{\textrm{(mixed)}}}
\newcommand{\sigmaxymixed}[0]{\mathbf{\Sigma}_{xy}^{\textrm{(mixed)}}}

\newcommand{\xpt}[0]{\mathbf{X}^{\textrm{(pt)}}}
\newcommand{\ypt}[0]{\mathbf{Y}^{\textrm{(pt)}}}
\newcommand{\sigmaxypt}[0]{\mathbf{\Sigma}_{xy}^{\textrm{(pt)}}}
\newcommand{\wf}[0]{\mathbf{W}_{1}}
\newcommand{\ws}[0]{\mathbf{W}_{2}}

\section{Theoretical Analysis}
\label{app:thy}
\subsection{Preliminaries and Setup}
    \paragraph{Model and data distribution} We consider a two-layer linear network $\theta= \mathbf{W_{1}} \mathbf{W_{2}} \mathbf{x}$ on a series of regression problems using the squared loss. In particular, the problems take the form of $\mathcal{L}_{t}(\theta) = \mathbf{E}_{x \sim \mathcal{D}_{\textrm{t}}}[||\theta x - \mathbf{A}^{t} x||_{2}^{2}]$, where $t \in \{\textcolor{softblue}{\textrm{pre}},\textcolor{softgreen}{\textrm{post}},\textcolor{softred}{\textrm{ft}}\}$ index the current stage of training. Here, $\mathcal{D}_{t}$ denotes the input distribution and the ground-truth outputs are generated as $\mathbf{A}^{t} \mathbf{X}$, where $\mathbf{X} \sim \mathcal{D}_{t}$. Following the analysis in \citet{springer2025overtrainedlanguagemodelsharder}, we consider the singular values and vectors of $\mathbf{A}^{t}$ as the learned features for the training task $t$.
    \begin{assumption}[Simultaneous Diagonalizability]
    There are orthonormal matrices $\mathbf{U}, \mathbf{V}$ such that for $t \in \{\textrm{pre},\textrm{post},\textrm{ft}\}$  we can write    
    $\mathbf{A}^{t} = \mathbf{U} \mathbf{\Sigma_{t}} \mathbf{V}^{\top}$, where all the $\mathbf{\Sigma_{t}}$ are diagonal matrices.
    \end{assumption}

In order to model transfer and interference between the distributions, we will next specify a structure on the relationships between the different features. We first assume the presence of \emph{invariant features}, capturing common linguistic capabilities that are broadly applicable across domains and tasks. Across these definitions, we assume a consistent indexing of the singular values (although the ordering of the singular values in terms of their magnitude may be different). In the following we will denote the singular values of the the task covariances interchangeably with the notations $\sigma^{t}_{i} = (\Sigma_{t})_{ii}$ to denote the ground-truth value of the feature.

\begin{definition}[Invariant Features] For $i \in [1, n-2k]$, we have that $(\mathbf{\Sigma}_{\textrm{pre}})_{ii} = (\mathbf{\Sigma_{\textrm{post}}})_{ii} = (\mathbf{\Sigma}_{\textrm{ft}})_{ii}$. For clarity and to emphasize their static nature, we will often denote the values of the invariant features as $\sigma^{\textrm{inv}}_{1},...,\sigma^{\textrm{inv}}_{n-2k}$. For conciseness, we will also use $d_{\textrm{invariant}} = n-2k$ to refer to the number of invariant features.
\end{definition}

In addition to these highly general features, we also consider the features through which the model may learn more domain specific information. We consider that such specialization can be implemented through one of two pathways:
\begin{definition}[Inconsistent Features] We define a feature (indexed by $i$) to be inconsistent if we have that $(\mathbf{\Sigma}_{\textrm{post}})_{ii} >  (\mathbf{\Sigma}_{\textrm{pre}})_{ii}$ and $(\mathbf{\Sigma}_{\textrm{post}})_{ii} - (\mathbf{\Sigma}_{\textrm{pre}})_{ii} > c_{\textrm{mis}}$
\end{definition}

Inconsistent features therefore incur a tradeoff between reducing loss on $\mathcal{D}_{\textrm{post}}$ and preserving performance on $\mathcal{D}_{\textrm{pre}}$. Finally, we introduce \textit{specialized features}, which do not incur such a tradeoff.

\begin{definition}[Specialized features] We consider feature $i$ is \textbf{specialized} if we have that $(\mathbf{V}^{\top})_{i} \mathbf{x} = 0$ and that $(\mathbf{\Sigma_{\textrm{post}}})_{ii} > 0$. For simplicity we will assume that all specialized features take the same value of  $(\mathbf{\Sigma_{\textrm{post}}})_{ii} = \beta$. We will also consider that $\beta < \frac{1}{2} c_{\textrm{mis}}$, which encodes that inconsistent features shift by a relatively large magnitude across the distributions.
\end{definition}

Intuitively, $\mathcal{D}_{\textrm{pre}}$ and $\mathcal{D}_{\textrm{ft}}$ have no covariance along the specialized feature directions. As we will show, this results in gradient steps taken along them causing no interference along these directions. However, as a result of their zero-covariance, these features are also impossible to learn without explicitly seeing the post-training data.

\textbf{Downstream Tuning Task} We consider that the downstream tuning task is relatively more similar to the pretraining task than the post-training task. As such, we consider that the inputs are sampled according to $x \sim \mathcal{N}(0,\mathbf{I}_{n-k})$ (i.e. it doesn't activate the specialized features). As previously, we have that the singular values corresponding to the invariant features remain constant. We will also consider that $\mathbf{A}^{\textrm{post}}$ and $\mathbf{A}^{\textrm{ft}}$ diverge on the inconsistent feature. Concretely we have that $(\mathbf{\Sigma}_{\textrm{post}})_{ii} >  (\mathbf{\Sigma}_{\textrm{ft}})_{ii}$ and $(\mathbf{\Sigma}_{\textrm{post}})_{ii} - (\mathbf{\Sigma}_{\textrm{ft}})_{ii} > c_{\textrm{mis}}$.

\paragraph{Mixed Training} We parameterize the mixed distribution by a parameter $\alpha$ and train on the distribution $\mathcal{D}_{\alpha, \textrm{mixed}} = (1-\alpha) \dweb + \alpha\dpost$. As all distributions in our setting have mean zero, we have that the covariance matrix of this mixture of Gaussian distributions is  $(1-\alpha) \mathbf{\Sigma_{\text{pre}}} + \alpha\mathbf{\Sigma}_{\textrm{post}}$.

\begin{assumption}[Invariant Features are High Magnitude] We consider that the invariant features are higher magnitude than the specialized features and the inconsistent features, concretely:
\begin{equation*}
   \sigma_{i}^{\textrm{pre}} > \sigma_{j}^{\textrm{pre}} 
\end{equation*}
$\forall i \in [0,d_{\textrm{invariant}}]$ and $\forall j \in (d_{\textrm{invariant}},n)$. We make a similar assumption on the relationship between the invariant features and the specialized features, concretely:
\begin{equation*}
\sigma_{i}^{\textrm{pre}} > \sigma_{j}^{\textrm{post}}
\end{equation*}
$\forall i \in [0,d_{\textrm{invariant}}]$ and $\forall j \in (d_{\textrm{invariant}},n)$. This intuitively encodes that the invariant features correspond to the strongest directions in the data.
\label{assmp:invfeatures}

\end{assumption}
\begin{assumption}[Sufficient Specialized Mixing] We assume that there exists $\alpha \in [0,1]$
\begin{equation*}
\alpha\beta >(1-\alpha)\sigma^{\textrm{pre}}_{i} + \alpha \sigma_{i}^{\textrm{post}} \,\,\, \forall i \in [d_{\textrm{invariant}}, d_{\textrm{invariant}}+k]
\end{equation*}
\label{asmp:sufficientspecializedmixing}
\label{assmp:suffmixing}
\end{assumption}

Intuitively, Assumption \ref{asmp:sufficientspecializedmixing} suggests that there exists a mixing ratio such that the mixing specialized features become more salient than the inconsistent features. However, this mixing ratio need not be high if the strength of the specialized feature is high in the covariance of $\dpost$.

\subsection{Analysis of Initial Pretraining}

Here, we will study the dynamics of the pretraining stage. We first introduce an important result on the sequential learning dynamics of features in two layer linear networks \citep{gidel2019implicitregularizationdiscretegradient}. For a given pretraining task where $\mathbf{X},\mathbf{Y}$ represent the inputs and outputs, respectively, we define that $\mathbf{\Sigma}_{xy} = \frac{1}{n} \mathbf{X}^{\top} \mathbf{Y}$ and $\mathbf{\Sigma}_{x} = \frac{1}{n} \mathbf{X}^{\top}\mathbf{X}$. We will write that $\mathbf{\Sigma}_{xy} = \sum_{i=1}^{R_{xy}} \sigma_{i} \mathbf{u}_{i}\mathbf{v_{i}}^{T}$, where $R_{xy}$ is the rank of $\mathbf{\Sigma_{xy}}$. We will also assume that: 

\begin{assumption}[Joint Decomposition] There exist orthogonal matrices $\mathbf{U},\mathbf{V}$ such that 
\begin{equation}
    \mathbf{\Sigma}_{xy} = \mathbf{U} \mathbf{D}_{xy}\mathbf{V}^{\top}, \mathbf{\Sigma}_{xx} = \mathbf{U}\mathbf{D}_{xx} \mathbf{U}^{\top}
\end{equation} and we will denote the singular values of $\mathbf{\Sigma}_{xy}$ as $\sigma_{1},...,\sigma_{R_{xy}}$ and the diagonal entries of $\mathbf{D}_{xx}$ as $\lambda_{1},...,\lambda_{R_{x}}=1$. 
\label{assmp:covariance_decomp}
\end{assumption}

Next, we will characterize the initialization scale of model before pretraining. Following \citep{springer2025overtrainedlanguagemodelsharder}, we have the following initialization:
\begin{assumption}[Pretrained Initialization Scale]
Let $(\mathbf{W}_{1}(0), \mathbf{W}_{2}(0))$ be the parameters at initialization. Then we have that $\mathbf{W}_{1}(0) = \mathbf{W}_{2}(0) = \exp({- \mathcal{T}) \mathbf{I_{d}}}$. 
\label{assmp:pretrained_initialization}
\end{assumption}

Essentially, Assumption \ref{assmp:pretrained_initialization} requires that the model parameters are close, but not exactly $0$ which yields \emph{sequential feature learning}. We next explicitly re-state the result from \cite{gidel2019implicitregularizationdiscretegradient}.
\begin{theorem}[Sequential Learning of Features \cite{gidel2019implicitregularizationdiscretegradient, springer2025overtrainedlanguagemodelsharder}] Suppose  $\mathbf{W}_{1},\mathbf{W}_{2}$ obey the initialization in Assumption \ref{assmp:pretrained_initialization} and the pretraining task obeys Assumption \ref{assmp:covariance_decomp}. Then there exist times $t_{1},...,t_{r}$ such that
\begin{equation*}
|| \mathbf{W}_{1}(t_{i}) -\mathbf{U}(\Sigma_{:i})^{\frac{1}{2}}||_{F} \leq \exp(-C\tau)
\end{equation*}
\begin{equation*}
    || \mathbf{W}_{2}(t_{i}) -(\Sigma_{:i})^{\frac{1}{2}} \mathbf{V}^{\top}||_{F} \leq \exp(-C\tau)
\end{equation*}
Where $\Sigma_{:i}$ is defined to be $\textrm{diag}(\sigma_{1},...\sigma_{i},0,...0)$, equivalently the rank $i$ approximation of $\textrm{diag}(\sigma_{1},...,\sigma_{R_{xy}})$.
\label{thm:seqlearning}
\end{theorem}
Conceptually,  Theorem \ref{thm:seqlearning} demonstrates that during the pretraining process, $\mathbf{W}_{1}\mathbf{W}_{2}$ learn features in decreasing order of their of the singular value of $\Sigma_{xy}$. Next, we will apply this result in order to compare the features learned during mixed and non-mixed pretraining. 

\begin{theorem}[Only Mixing Learns Specialized Features] Let $\theta^{\textrm{gen}}(t)= \mathbf{W}^{\textrm{gen}}_{1}(t)\mathbf{W}^{\textrm{gen}}_{2}(t)$ be the parameters learned when pretraining  only on $\dgen$ and $\theta^{\textrm{mixed}}(t)= \mathbf{W}^{\textrm{mixed}}_{1}(t)\mathbf{W}^{\textrm{mixed}}_{2}(t)$. Denote $\mathbf{u}_{\textrm{spec}}, \mathbf{v}_{\textrm{spec}}$ to be the right and left singular vectors corresponding to the \textit{specialized feature}. Then, there exists a time $t$ such that 

\begin{align*}
        ||\mathbf{W}_{1}^{\textrm{(unmixed)}}-\mathbf{U}(\mathbf{\Sigma}^{\textrm{(unmixed})})^{\frac{1}{2}}||_{F} \leq \exp(-C\tau)\\
    ||\mathbf{W}_{2}^{\textrm{(unmixed)}}-(\mathbf{\Sigma}^{\textrm{(unmixed})})^{\frac{1}{2}}\mathbf{V}^{\top}||_{F} \leq \exp(-C\tau)\\
        ||\mathbf{W}_{1}^{\textrm{(mixed)}}-\mathbf{U}(\mathbf{\Sigma}^{\textrm{(mixed})})^{\frac{1}{2}}||_{F} \leq \exp(-C\tau)\\
    ||\mathbf{W}_{2}^{\textrm{(mixed)}}-(\mathbf{\Sigma}^{\textrm{(mixed})})^{\frac{1}{2}}\mathbf{V}^{\top}||_{F} \leq \exp(-C\tau)
\end{align*}
where $\mathbf{\Sigma}_{\textrm{mixed}} = \textrm{diag}(\sigma^{\textrm{inv}}_{1},...\sigma_{n-2k}^{\textrm{inv}},\mathbf{0}_{k},\alpha\beta,...,\alpha\beta)$ and $\mathbf{\Sigma}_{\textrm{unmixed}} = \textrm{diag}(\sigma_{1}^{\textrm{inv}},...,\sigma_{n-2k}^{\textrm{inv}},\sigma_{1}^{\textrm{post}},...,\sigma^{\textrm{post}}_{k}, \mathbf{0}_{k})$
\label{thm:dfffeatures}
\end{theorem}
\begin{proof}
This is a relatively straightforward application of Theorem~\ref{thm:seqlearning}. Denote $(\xmixed, \ymixed)$ as the data used for mixed pretraining and $\sigmaxymixed = \frac{1}{n} (\xmixed)^{\top}\ymixed$. We have that $\mathbf{\Sigma}_{xy}^{\textrm{(mixed)}} = (1-\alpha)\mathbf{\Sigma}^{\textrm{pre}}_{xy} + \alpha \mathbf{\Sigma}^{\textrm{spec}}_{xy}$ which follows from the fact that the $\Sigma_{xy}$ are submatrices of the covariance matrix of Gaussian random vectors. By Theorem \ref{thm:seqlearning}, we have that the features are learned in order of the singular values of $\mathbf{\Sigma_{xy}}$. By Assumptions \ref{assmp:invfeatures} and \ref{assmp:suffmixing}, we have that the top $n-k$ singular values of $\mathbf{\Sigma}_{xy}^{\textrm{(mixed)}}$ are the $n-2k$ shared features and the $k$ specialized features. Define $\mathbf{\Sigma}^{\textrm{(mixed})}_{:n-k} = \textrm{diag}(\sigma^{\textrm{inv}}_{1},...\sigma_{k}^{\textrm{inv}},\mathbf{0}_{k},\alpha\beta,...,\alpha\beta)$. Applying Theorem~\ref{thm:seqlearning}, we have that
\begin{align*}
    ||\mathbf{W}_{1}^{\textrm{(mixed)}}-\mathbf{U}(\mathbf{\Sigma}^{\textrm{(mixed})}_{:n-k})^{\frac{1}{2}}||_{F} \leq \exp(-C\tau)\\
    ||\mathbf{W}_{2}^{\textrm{(mixed)}}-(\mathbf{\Sigma}^{\textrm{(mixed})}_{:n-k})^{\frac{1}{2}}\mathbf{V}^{\top}||_{F} \leq \exp(-C\tau)
\end{align*}

Repeating this analysis for unmixed training, we have that the top $n-k$ singular values of $\mathbf{\Sigma}^{\textrm{(gen)}}_{xy}$ are the $n-2k$ shared features are the $k$ inconsistent features. We can define $\mathbf{\Sigma}^{\textrm{(unmixed)}}_{:n-k} = \textrm{diag}(\sigma^{\textrm{inv}}_{1},...\sigma_{k}^{\textrm{inv}},\sigma^{\textrm{post}}_{1},...,\sigma^{\textrm{post}}_{k}, \mathbf{0}_{k})$. Similarly, by applying Theorem \ref{thm:seqlearning}, we have that 
\begin{align*}
        ||\mathbf{W}_{1}^{\textrm{(unmixed)}}-\mathbf{U}(\mathbf{\Sigma}^{\textrm{(unmixed})}_{:n-k})^{\frac{1}{2}}||_{F} \leq \exp(-C\tau)\\
    ||\mathbf{W}_{2}^{\textrm{(unmixed)}}-(\mathbf{\Sigma}^{\textrm{(unmixed})}_{:n-k})^{\frac{1}{2}}\mathbf{V}^{\top}||_{F} \leq \exp(-C\tau)
\end{align*}
\end{proof}

Intuitively, our result in Theorem \ref{thm:dfffeatures} demonstrates that mixing $\dpost$ during pretraining results in a pretrained initialization that has different features. Mixing learns the specialized features, while not mixing learns only the inconsistent features. In the following, we will examine the impact that these different features have on the retention of $\dpost$ during subsequent training.
\subsection{Analysis of Post-Training}
We now study the dynamics of the post-training process. To formalize the post-training process, we first examine the dynamics beginning from the idealized pretraining initialization (as performed by \cite{springer2025overtrainedlanguagemodelsharder}). We perform the post-training stage on the regularized loss $\mathbf{E}[||\theta x-\mathbf{A}^{\textrm{sp}}x||_{F}^{2}] + \lambda ||\theta-\theta_{0}||_{F}^{2}$. Observe that because $x \sim \mathcal{N}(0, \mathbf{I}_{d})$, this is equivalent to $||\theta-\mathbf{A}^{\textrm{sp}}||_{F}$. We follow the assumptions on the regularity of fine-tuning established in \cite{springer2025overtrainedlanguagemodelsharder}.

\begin{assumption}[Bound on Parameters Throughout Training]
\begin{align*}
    ||\hat{\mathbf{W}}^{(\textrm{mixed})}_{1}||_{\textrm{op}} \leq \sqrt{\Gamma}\\
    ||\hat{\mathbf{W}}^{(\textrm{mixed})}_{1}||_{\textrm{op}} \leq \sqrt{\Gamma}\\
    ||\hat{\mathbf{W}}^{(\textrm{unmixed})}_{1}||_{\textrm{op}} \leq \sqrt{\Gamma}\\
    ||\hat{\mathbf{W}}^{(\textrm{unmixed})}_{1}||_{\textrm{op}} \leq \sqrt{\Gamma}
\end{align*}
\end{assumption}
Moreover, we assume that the regularization strength and the learning rates are likewise bounded.

\begin{assumption}[Bound on Learning Rate]
    \begin{equation*}
        4\eta (\lambda +2) \Gamma <1
    \end{equation*}
\end{assumption}

\paragraph{Idealized Pretraining Initialization} We denote the ideal initialization parameters for the mixed and unmixed cases $(\hat{\wf}(0), \hat{\ws}(0))$.

\begin{align*}
\hat{\wf^{\textrm{(mixed)}}}(0) = \mathbf{U} (\Sigma^{\textrm{mixed}}_{:n-k})^{\frac{1}{2}}\\
    \hat{\ws^{\textrm{(mixed)}}}(0) = (\Sigma^{\textrm{mixed}}_{:n-k})^{\frac{1}{2}} \mathbf{V}^{\top}
\end{align*}

Similarly, we have the following idealized initialization for the unmixed initialization: 
\begin{align*}
\hat{\wf^{\textrm{(unmixed)}}}(0) = \mathbf{U} (\Sigma^{\textrm{unmixed}}_{:n-k})^{\frac{1}{2}}\\
    \hat{\ws^{\textrm{(unmixed)}}}(0) = (\Sigma^{\textrm{unmixed}}_{:n-k})^{\frac{1}{2}} \mathbf{V}^{\top}
\end{align*}

In the idealized setting, we can track the evolution of each singular value independently. In particular, we have the following update rules as derived in \cite{springer2025overtrainedlanguagemodelsharder} (where we denote $\sigma_{i}^{\textrm{spec}}$ as the $i$-th singular value of  $\mathbf{A}^{\textrm{spec}}$ and likewise for $\sigma^{\textrm{(un)mixed}}_{i}(t)$ as the $i$-th singular value  at step $t$). In what follows, we will suppress the superscript for compactness:
\begin{equation}
\sigma_{i}^{}(t+1) = \sigma_{i}(t) - 2 \eta\sigma_{i}(t)(\sigma_{i}(t)^{2} - (\sigma_{\textrm{spec},i})^{2})+ 2\eta\lambda (\sigma_{i}(t)^{2} - \sigma_{i}(0)^{2} )
\end{equation}
As a result, note that when $\sigma_{i}^{\textrm{(un)mixed}}(0) = 0$, $\sigma_{i}^{\textrm{(un)mixed}}(t) = 0$ for all $t$.

Next,we will study the dynamics of the non-zero singular values (Lemma A.11 \cite{springer2025overtrainedlanguagemodelsharder}). We will assume that post-training is performed for a sufficient number of steps.

\begin{assumption}[Sufficient Post-Training Steps]
We have that the number of post-training steps (denoted by $K$) satisfies $K \geq \frac{1}{\lambda c_{\textrm{min}}} \log \frac{100 \Gamma}{\epsilon}$, for a constant $\epsilon < \frac{c_{\textrm{mis}}}{2c_{\textrm{mis}} - 4\beta}$ and where $c_{\textrm{min}} = \min \{(\mathbf{\Sigma}_{\textrm{post}})_{ii} | (\mathbf{\Sigma}_{\textrm{post}})_{ii} \neq 0\} $ --- that is the minimum, non-zero singular value.
    
\end{assumption}

Given these technical conditions, we now state a general result (adapted for our setting from \cite{springer2025overtrainedlanguagemodelsharder}). 

\begin{lemma}
When training on $\dpost$ with infinite batch size from the ideal pretraining initialization and taking sufficient number of steps $K$, for all $i \in \textrm{rank}(\theta_{n}(0))$, we have that 
\begin{align}
|(\mathbf{U}^{\top} \hat{\theta}^{\textrm{(un)mixed}}(t) \mathbf{V})_{ii} - (\mathbf{\Sigma}_{\textrm{post}})_{ii}| \leq \epsilon 
\end{align}
where $\mathbf{\Sigma}_{\textrm{post}}$ is such that $\mathbf{A}_{\textrm{post}} = \mathbf{U} \mathbf{\Sigma}_{\textrm{post}} \mathbf{V}^{\top}$.
\label{lemm:finetuning}
\end{lemma}

Now, we will define the matrices $\mathbf{\Sigma}^{\textrm{shared,post}} = \textrm{diag}(\sigma^{\textrm{inv}}_{1},...\sigma^{\textrm{inv}}_{n-2k}, \sigma^{\textrm{post}}_{d_{\textrm{invariant}}+1},...\sigma^{\textrm{post}}_{d_{\textrm{invariant}}+k},\mathbf{0}_{k})$ and $\mathbf{\Sigma}^{\textrm{spec,post}} = \textrm{diag}(\sigma^{\textrm{inv}}_{1},...\sigma^{\textrm{inv}}_{n-2k}, \mathbf{0}_{k},\sigma^{\textrm{post}}_{d_{\textrm{invariant}}+k+1},...\sigma^{\textrm{post}}_{d_{\textrm{invariant}}+2k})$, Intuitively, $\mathbf{\Sigma}^{\textrm{shared,post}}$ lowers loss on $\dpost$ by shifting the values on the shared features, while $\mathbf{\Sigma}^{\textrm{spec,post}}$ accomplishes this by modifying the unique features. We are now ready to state the main theorem.

\begin{theorem}[Post-training on $\theta^{\textrm{mixed}}$ versus $\theta^{\textrm{unmixed}}$] Let $\theta^{\textrm{unmixed}}(K)$ denote the parameters after training the idealized unmixed initialization starting for $K$ steps and let $\theta^{\textrm{mixed}}(K)$ be the same starting from the idealized mixed checkpoint. Then we have

\begin{align*}
||\mathbf{U}^{\top} \theta^{\textrm{mixed}}\mathbf{V} - \mathbf{\Sigma}^{\textrm{spec,post}}||_{\textrm{op} }\leq \epsilon \\
||\mathbf{U}^{\top} \theta^{\textrm{unmixed}}\mathbf{V} - \mathbf{\Sigma}^{\textrm{shared,post}}||_{\textrm{op} }\leq \epsilon 
\end{align*}
\end{theorem}
\begin{proof}
This theorem follows by noting that under the idealized pretraining initialization any singular value that is $0$ at initialization remains that way during the entire optimization trajectory. Note that the $0$ singular values of $\mathbf{U}^{\top}\theta^{\textrm{mixed}}\mathbf{V}$ coincide with $\mathbf{\Sigma}^{\textrm{spec, post}}$ (the specialized features) and likewise $\mathbf{U}^{\top}\theta^{\textrm{unmixed}}\mathbf{V}$ coincide with $\mathbf{\Sigma}^{\textrm{shared, post}}$ (the inconsistent features).

This implies that we have 
\begin{align*}
\max_{i \in [1,n]} |(\mathbf{U}^{\top} \theta^{\textrm{mixed}}\mathbf{V})_{ii} - (\mathbf{\Sigma}^{\textrm{spec,post}})_{ii}| \leq \max_{i \in \textrm{rank}(\theta)}|(\mathbf{U}^{\top} \theta^{\textrm{mixed}}\mathbf{V})_{ii} - (\mathbf{\Sigma}^{\textrm{spec,post}})_{ii}|\\
\max_{i \in [1,n]} |(\mathbf{U}^{\top} \theta^{\textrm{unmixed}}\mathbf{V})_{ii} - (\mathbf{\Sigma}^{\textrm{shared,post}})_{ii}| \leq \max_{i \in \textrm{rank}(\theta)}|(\mathbf{U}^{\top} \theta^{\textrm{mixed}}\mathbf{V})_{ii} - (\mathbf{\Sigma}^{\textrm{shared,post}})_{ii}|
\end{align*}
Now, applying the result from Lemma~\ref{lemm:finetuning} yields the desired claim.
\end{proof}

\subsection{Analysis of Downstream Adaptation}
In the previous section, we characterized the impact of post-training from a mixed versus an unmixed initialization, demonstrating that different features are used to minimize the loss on $\dpost$. In this section, we study how these different features impact the ultimate retention of $\dpost$. We first establish that the singular values corresponding to directions in which there is no covariance remain unchanged throughout the downstream fine-tuning stage. 
\begin{lemma}
Consider performing downstream unregularized fine-tuning on $\dft$. If $x\sim \dft$ has $0$ covariance along a singular direction, the corresponding singular value remains unchanged throughout downstream adaptation.
\label{lemm:orthogonalfeatures}
\end{lemma}
\begin{proof}
To see this, note that the gradient updates for $\mathbf{W}_{1}$ and $\mathbf{W}_{2}$ take the following form:
\begin{align*}
\mathbf{W}_{1}(k+1) = \mathbf{W}_{1}(k) - 2 \eta(\mathbf{W}_{1}(k)\mathbf{W}_{2}(k)-\mathbf{A}^{\textrm{spec}})\Sigma_{x} \mathbf{W}
_{2}^{\top}\\
\mathbf{W}_{2}(k+1) = \mathbf{W}_{2}(k) - 2 \eta\mathbf{W}_{1}(k)\Sigma_{x}(\mathbf{W}_{1}(k)\mathbf{W}_{2}(k)-\mathbf{A}^{\textrm{spec}})
\end{align*}
Here, we have that $\Sigma_{x}$ denotes the covariance of the input data $x$. Thus, along any singular direction in which the data has $0$ variance, the $\mathbf{\Sigma}_{x}$ term will project the gradient to $0$. Therefore, the singular values on such directions must also remain unchanged.
\end{proof}

We consider performing downstream adaptation by taking steps using unregularized gradient descent on $\mathcal{D}_{\textrm{ft}}$ and show the following result.

\begin{theorem}
Consider performing $K$ steps of gradient descent  on the downstream finetuning dataset beginning from intializations $\theta^{\textrm{post, mixed}}(K)$ and let $\theta^{\textrm{FT,mixed}}(K)$, $\theta^{\textrm{FT,unmixed}}(K)$, and $\theta^{\textrm{post, unmixed}}(K)$ denote the final parameters. Let $\Delta_{\textrm{unmixed}} = \mathcal{L}(\theta^{\textrm{FT, unmixed}};\dpost) -\mathcal{L}(\theta^{\textrm{post, unmixed}}(K);\dpost)$ and likewise $\Delta_{\textrm{mixed}} = \mathcal{L}(\theta^{\textrm{FT, mixed}};\dpost) -\mathcal{L}(\theta^{\textrm{post,mixed}}(K);\dpost)$. Then we have that $\Delta_{\textrm{unmixed}} > \Delta_{\textrm{mixed}}.$
\end{theorem}
\begin{proof}
As the invariant features  take the same values, they will not move during the downstream adaptation. Moreover, due to the Lemma \ref{lemm:orthogonalfeatures}, we also have that the specialized features will not change during the the downstream fine-tuning. This implies that $\Delta_{\textrm{mixed}} = 0$. Next we will examine the changes induced by downstream training on the unmixed models. Observe that we have that the $\mathcal{L}(\theta;\dpost)=||\theta - \mathbf{A}^{\textrm{spec}}||_{F}^{2}$. We will define the following matrices $\Sigma_{\textrm{FT}}^{\textrm{(unmixed)}} = \textrm{diag}(\sigma_{1}^{\textrm{inv}},....,\sigma^{\textrm{inv}}_{d_{\textrm{invariant}}},\sigma^{\textrm{ft}}_{d_{\textrm{invariant}}+1},...\sigma^{\textrm{ft}}_{\textrm{d}_{\textrm{invariant}}+k} \mathbf{0}_{k})$ and note that Lemma \ref{lemm:finetuning} gives us that

\begin{align*}
||\mathbf{U}^{\top}\theta^{\textrm{(unmixed)}}_{\textrm{FT}}\mathbf{V}- \Sigma_{\textrm{FT}}^{\textrm{(unmixed)}}||_{\textrm{op}} \leq\epsilon
\end{align*}

The loss  function we use here is simply the squared difference of the singular values. Thus,we can upper bound:
\begin{align*}
\mathcal{L}(\theta_{\textrm{post}}^{\textrm{unmixed}};\dpost)\leq k(\beta + \epsilon)^{2}+k\epsilon^{2}
\end{align*}
and likewise lower bound
\begin{align*}
\mathcal{L}(\theta_{\textrm{ft}}^{\textrm{unmixed}};\dpost)\geq k(\beta -\epsilon)^{2}+k(c_\textrm{mis}-\epsilon)^2
\end{align*}
Then,we can lower bound $\Delta_{\textrm{unmixed}} \geq k[-(\beta+\epsilon)^{2}+(\beta-\epsilon)^{2}] + k[(c_{\textrm{mis}}-\epsilon)^{2} - \epsilon^{2}] = k[-4\beta\epsilon -2c_{\textrm{mis}}\epsilon + c_{mis}^{2}]$. From the condition that $\beta<\frac{1}{2} c_{\textrm{mis}}$, the positivity of $(c_{\textrm{mis}})^{2}$, and the condition on $\epsilon$, we thus have that $\Delta_{\textrm{unmixed}}>0$, which is what we wanted to show.

\end{proof}

\end{document}